\documentclass{article}

\PassOptionsToPackage{numbers, compress}{natbib}

\usepackage{amsmath}
\usepackage{amsthm}
\usepackage{enumitem}
\usepackage{graphicx}
\usepackage{wrapfig}
\usepackage{xcolor}
\usepackage{subcaption}
\theoremstyle{plain}
\newtheorem{proposition}{Proposition}
\newtheorem{remark}{Remark}

\newtheorem{theorem}{Theorem}

\usepackage[hidelinks]{hyperref}
\definecolor{citeblue}{HTML}{1565C0}
\hypersetup{colorlinks=true, linkcolor=citeblue, citecolor=citeblue, urlcolor=citeblue}

\definecolor{revblue}{RGB}{0, 90, 180}
\definecolor{revred}{RGB}{180, 30, 30}

\usepackage[preprint]{neurips_2026}

\usepackage[utf8]{inputenc} 
\usepackage[T1]{fontenc}    
\usepackage{hyperref}       
\usepackage{url}            
\usepackage{booktabs}       
\usepackage{amsfonts}       
\usepackage{nicefrac}       
\usepackage{microtype}      
\usepackage{xcolor} 
\usepackage[normalem]{ulem}

\title{Stochastic Adaptive Fourier Decomposition for Operator Learning}

\author{%
   Pengqing Shi \\
   University of Sydney \\
  \texttt{pengqing.shi@sydney.edu.au} \\
   \And
   Liming Zhang \\
   University of Macau \\
   \texttt{lmzhang@um.edu.mo} \\
   \And
   Tao Qian \\
   Macau University of Science and Technology \\
   \texttt{tqian@must.edu.mo} \\
   \And
   Stephen Tierney \\
   University of Sydney \\
   \texttt{stephen.tierney@sydney.edu.au} \\
   \And
   Jie Yin \\
   University of Sydney \\
   \texttt{jie.yin@sydney.edu.au} \\
   \And
   Junbin Gao \\
   University of Sydney \\
   \texttt{junbin.gao@sydney.edu.au} \\
}

\begin{document}

\maketitle

\begin{abstract}
Fourier Neural Operators (FNOs) offer an efficient paradigm for solving partial differential equations (PDEs). However, FNOs rely on a fixed Fourier basis and hard frequency truncation, which inherently limit their ability to model non-periodic, localized, and fine-scale solution structures. We propose the Stochastic Adaptive Fourier Decomposition Neural Operator (SAFDNO), a spectral neural operator that replaces predefined Fourier modes with an adaptive Takenaka-Malmquist (TM) orthonormal system derived from the theory of Stochastic Adaptive Fourier Decomposition (SAFD). Instead of performing an expensive greedy pole search in classical SAFD, SAFDNO amortizes stochastic pole selection through a neural pole predictor and constructs the adaptive TM system directly from latent features. The resulting operator performs learned filtering on coefficients of analytic branches in the Hardy space with respect to the input-adaptive TM system, preserving the global receptive field of spectral operators while providing a more flexible representation than using fixed spectral bases. Across nine PDE benchmark problems, SAFDNO achieves the best performance on all six regular-grid problems among strong neural operator baselines, with especially notable gains on Darcy, Burgers, and Navier-Stokes, where fixed Fourier modes are often less effective at modeling localized oscillations, sharp transitions, and multiscale structures. SAFDNO also exhibits stronger zero-shot super-resolution performance and shows less performance degradation when deployed on finer discretizations. These results suggest that the input-adaptive TM system provides a promising alternative to fixed spectral representations for neural operator learning.
\end{abstract}

\section{Introduction}\label{sec:intro}
Partial Differential Equations (PDEs) are fundamental mathematical tools for describing and modeling a wide range of physical and engineering systems, capturing how quantities evolve over time and space. Although traditional numerical methods are accurate,  they can be computationally expensive when solving PDEs on complex geometries. Operator learning, which aims to learn the mappings between infinite-dimensional function spaces, offers a more computationally efficient approach for solving various families of PDEs \citep{kovachki2023neuraloperator}.

The Fourier Neural Operator (FNO) \citep{fno} is an effective and efficient operator learning framework that represents a global convolution operator, equivalently a translation-invariant integral kernel, in the Fourier domain through learned Fourier multipliers. In practice, it applies the fast Fourier transform (FFT) with frequency truncation, performs a convolution operator on the retained low-frequency Fourier coefficients, and then uses the inverse FFT (iFFT) to transform the representation back to the physics domain. This architecture design equips FNO with efficient global receptive fields and discretization-invariant learning, allowing the learned operator to generalize across different discretizations. In particular, FNO can be trained on coarse grids and directly generalize to finer grids, thereby enabling zero-shot super-resolution.

\begin{figure}[htbp]
    \centering
    \includegraphics[width=\linewidth]{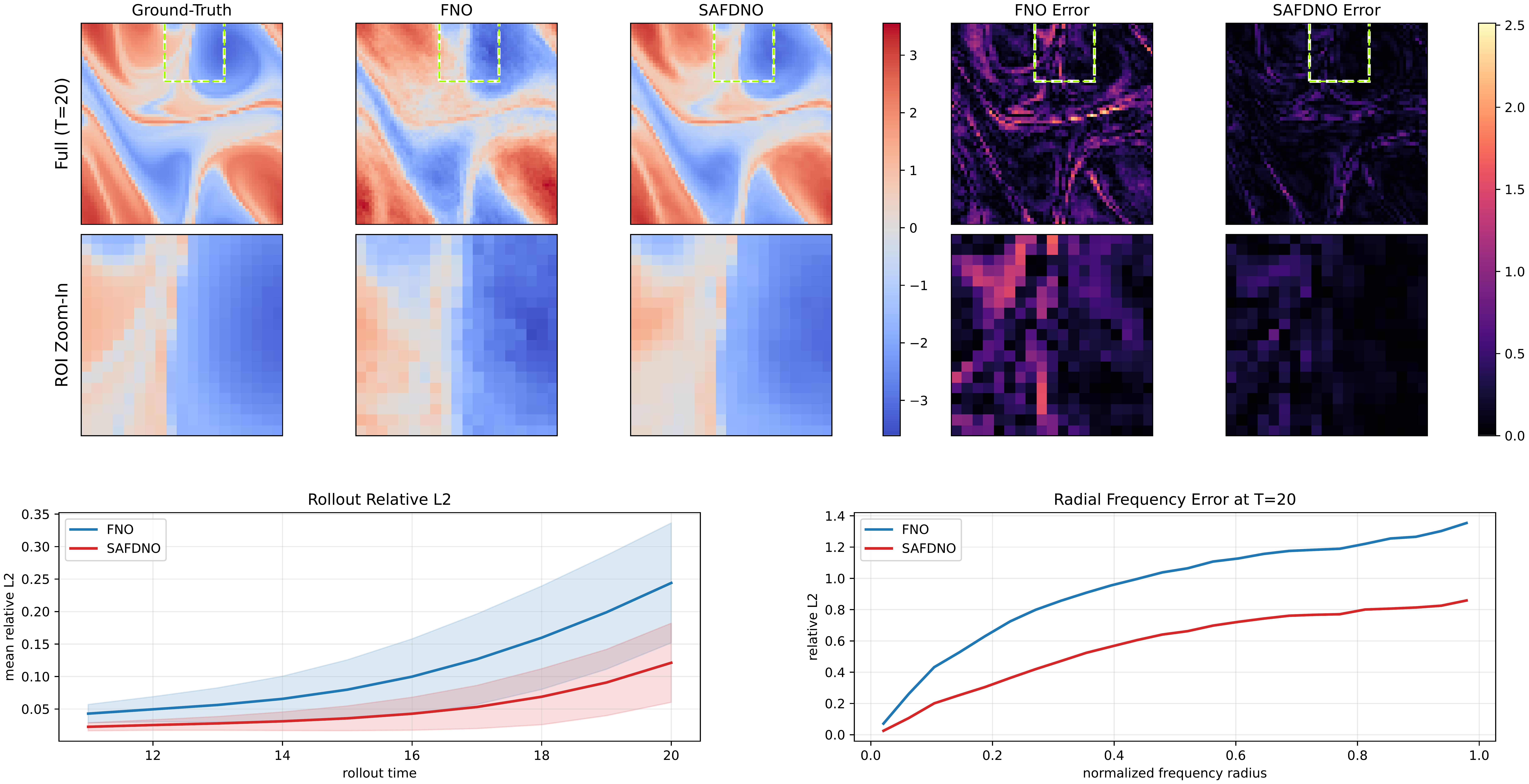}
    \caption{Performance on Navier-Stokes. Highlighted region indicates a representative area where SAFDNO handles multiscale structures better than FNO. Bottom plots show lower prediction errors for SAFDNO, especially for longer rollouts and higher frequencies.}
    \label{fig:multiscale_ns}
\end{figure}

Nevertheless, the use of frequency truncation and a fixed Fourier transform in FNO may hinder the expressive power and performance on certain PDEs. Specifically, frequency truncation may violate the equivalence of continuous and discrete representations in neural operators, thereby introducing aliasing effects \citep{gao2025discretizationinvariance}. In addition, truncating higher-frequency modes may induce spectral bias and limit the ability to capture fine-scale solution structures, particularly in PDEs with strong localized high-frequency oscillations \citep{amfno_xiao2024}. Moreover, the Fourier transform used in FNO inherently relies on the assumptions of Fourier analysis. As a result, dependence on a fixed Fourier basis may be inefficient for PDEs with non-periodic patterns, localized fine-scale structures, sharp discontinuities, and non-stationary dynamics \citep{pordanesh2025hilbertno,tang2025geometric}.

Recent studies have sought to address these limitations from several directions. Some extend FNO within the Fourier framework by improving Fourier-domain kernel operator layers to alleviate the spectral bias induced by fixed mode truncation \citep{amfno_xiao2024}, enriching Fourier layers with local convolution pathways to improve local and high-frequency representation \citep{wen2022u}, or extending FNO to irregular geometries through learned deformations \citep{geofno}. Others replace the Fourier transform with alternative predefined transforms or bases, such as wavelet, Laplace, or Hilbert representations, to better capture multiscale structures, transient signals, or phase information \citep{pordanesh2025hilbertno,tang2025geometric,CaoQianying2024Lnof,mwt,wno}. However, most existing approaches still rely on fixed spectral representations, rather than learning an input-adaptive basis from data, which may still inherently constrain the ability to learn non-periodic, localized high-frequency, non-stationary, and multiscale components.

To overcome the aforementioned limitations, we introduce the Stochastic Adaptive Fourier Decomposition Neural Operator (SAFDNO), motivated by the framework of Stochastic Adaptive Fourier Decomposition (SAFD) \citep{qian2022sparse,AFD_Qian,2DAFD_Qian,2DAFD_Gao_Qian,qian2011adaptive}. Unlike the Fourier transform, which expands signals over a fixed set of global sinusoidal bases, SAFD adaptively selects rational orthogonal functions associated with a Takenaka-Malmquist (TM) system according to input signals \citep{qian2022sparse,AFD_Qian}. Due to this adaptive mechanism, the SAFD method often exhibits faster energy convergence and can capture oscillatory structures and higher-frequency components more flexibly and efficiently than the fixed Fourier transform. Moreover, the SAFD method generally requires fewer poles than Fourier modes to achieve a comparable approximation error. These properties suggest that SAFD can provide a more flexible and parsimonious representation for PDE solution operators. As the representative example shown in Figure~\ref{fig:multiscale_ns}, Navier–Stokes turbulence manifests localized and multiscale dynamics, energy resides at large scales, while enstrophy concentrates on thin filamentary sets; our SAFDNO shows a stronger performance on learning non-stationary, high-frequency, and multiscale patterns.

Building on this perspective, our SAFDNO learns an input-dependent TM system representation directly from data, instead of relying on a predefined transform with explicit frequency truncation. In contrast to fixed spectral bases, the learned TM system adapts to input signals and allocates representation capacity adaptively according to pole placements, yielding a more flexible representation of the underlying structure for neural operator learning. Our main contributions are summarized as follows:
\begin{itemize}[leftmargin=*]
    \item We propose the Stochastic Adaptive Fourier Decomposition Neural Operator (SAFDNO), which incorporates 
    the adaptive Takenaka-Malmquist (TM) orthonormal system and Stochastic Adaptive Fourier Decomposition (SAFD) into spectral neural operator learning.
    \item We provide a theoretical analysis of the representation capacity of SAFD, proving its advantages over fixed Fourier bases on PDEs with localized and multiscale structures. The effectiveness of SAFDNO is empirically validated on multiple PDE benchmarks against strong baselines.  
\end{itemize}

\section{Related Work}

\paragraph{Operator learning.}
Operator learning is a data-driven framework that focuses on learning function-to-function mappings. DeepONet \citep{lu2021learning} is a representative architecture that parameterizes operators through branch and trunk networks. The FNO \citep{fno} parameterizes the global convolution by using a translation-invariant integral kernel in the Fourier domain and enables discretization-invariance. 
Attention-based operators, such as OFormer \citep{oformer} and GNOT \citep{gnot}, improve flexibility on irregular grids and arbitrary query locations. Our work is closely related to spectral neural operators.

\paragraph{Spectral neural operators.}
Building on FNO, a substantial body of work has advanced Fourier-based operator learning, yet still relies on the fixed Fourier transform. Geo-FNO \citep{geofno} extends FNO to arbitrary geometries by learning a deformation from the physical domain to a latent regular grid. F-FNO \citep{ffno} improves scalability through factorized spectral layers. AM-FNO \citep{amfno_xiao2024} alleviates the spectral bias caused by frequency truncation by using an amortized neural network to generate learnable weights for kernel parameterization, thereby allowing 
the model to accommodate arbitrarily many modes with a fixed number of parameters. DSFNO \citep{gao2025dynamic} further enhances expressiveness through input-dependent dynamic kernels and interactions among Fourier modes. However, these methods are intrinsically constrained by the inherent limitations of Fourier analysis and are therefore less suitable for capturing multiscale, spatially localized features or non-periodic structures. This limitation motivates a reconsideration of the underlying transform or basis. Wavelet-based methods, such as WNO \citep{wno}, exploit frequency localization to better capture spatially localized and fine-scale patterns. Multiwavelet-based operator learning projects the operator kernel onto multiwavelet polynomial bases and learns complex dependencies across multiple scales \citep{mwt}. The Laplace neural operator \citep{CaoQianying2024Lnof} represents the kernel integral operator in the Laplace domain using a pole-residue formulation between the input and solution spaces, thereby improving the modeling of non-periodic signals. Nevertheless, most existing works predominately rely on predefined transforms or fixed basis systems that are independent of the input, whereas input-adaptive basis construction remains relatively underexplored in operator learning.

\paragraph{Adaptive Fourier Decomposition.} Adaptive Fourier Decomposition (AFD) and its stochastic variant (SAFD) were originally introduced by Qian et al. \citep{qian2022sparse,AFD_Qian,qian2011adaptive} as signal analysis methods that employ adaptive rational orthonormal systems and have been shown to be effective in analyzing oscillatory and non-stationary signals. In machine learning, AFD has been used for classification \citep{TanChunyu2019ANSL} and feature extraction \citep{TanChunyu2021Anfr}, and SAFD has demonstrated strong compression performance in image analysis \citep{10416200,cheng2023dual,cheng2025deep} and video applications \citep{dai2022image,dai2022joint}. In operator learning, the Adaptive Mamba Operator (AMO) \citep{song2026adaptive} incorporates AFD by reparameterizing transfer functions as TM basis functions. Our SAFDNO learns the operator directly in the adaptive SAFD coefficient space.

\section{Preliminaries}
\subsection{Operator Learning}
Let $\Omega \subset \mathbb{R}^d$ be a bounded domain, $\mathcal{A} = \mathcal{A}(\Omega; \mathbb{R}^{d_a})$ be the input function space, and $\mathcal{U}=\mathcal{U}(\Omega; \mathbb{R}^{d_u})$ be the solution function space. Given a set of training pairs $\{ (a_i, u_i) \}_{i=1}^N \subset \mathcal{A} \times \mathcal{U}$, operator learning seeks a parametric map $\mathcal{G}_\theta: \mathcal{A} \rightarrow \mathcal{U}$ by solving
\begin{equation}
    \min_{\theta \in \mathbb{R}^p} \frac1N \sum_{i=1}^N \ell \big(\mathcal{G}_\theta(a_i), u_i \big). \label{eq:op_obj}
\end{equation}
In this work, we represent the overall architecture of a neural operator as compositions (see \citep{kovachki2023neuraloperator}),
\begin{equation}
    \mathcal{G}_\theta := \mathcal{P} \circ \mathcal{L}^{L} \circ \cdots \circ \mathcal{L}^{1} \circ \mathcal{Q},
    \label{eq:no_comp}
\end{equation}
where $\mathcal{Q}: \mathcal{A}(\Omega; \mathbb{R}^{d_a}) \rightarrow \mathcal{V}(\Omega; \mathbb{R}^{d_v})$ is the lifting layer, $\mathcal{L}^\ell: \mathcal{V}(\Omega; \mathbb{R}^{d_v}) \rightarrow \mathcal{V}(\Omega; \mathbb{R}^{d_v})$ for $\ell=1,\dots,L$ are operator blocks acting on a latent representation, and $\mathcal{P}: \mathcal{V}(\Omega; \mathbb{R}^{d_v}) \rightarrow \mathcal{U}(\Omega; \mathbb{R}^{d_u})$ is the projection layer.

Each $\mathcal{L}^{l}$ for $l=1, 2, \dots, L$ is generally defined as
\begin{equation}
  v^{(l+1)}(x) = \sigma \Big( W^{(l)} v^{(l)}(x) + \int_\Omega \kappa_\phi \big(x, y, a(x), a(y)\big) v^{(l)}(y) dy \Big) \quad \forall x \in \Omega.
\end{equation}
FNO \citep{fno} uses global convolution in the Fourier domain, and its kernel integral operator satisfies translation invariance $\kappa_\phi (x, y, a(x), a(y)) =\kappa_\phi (x - y)$. Under the translation-invariant assumption, the kernel operator becomes a convolution operator such that $(K_\phi v^{(l)})(x)=\int_\Omega \kappa_\phi (x - y) v^{(l)}(y) d y = \mathcal{F}^{-1} ( \mathcal{F}(\kappa_\phi) \cdot \mathcal{F}(v^{(l)}))(x)$, $\forall x \in \Omega$. With the Fourier integral operator, each FNO layer is formulated as
\begin{equation}\label{eq:FNO_layer}
    v^{(l+1)}(x) = \sigma \Big( W^{(l)} v^{(l)} (x) + \mathcal{F}^{-1} \big(\mathcal{R}_\phi^{l}(k) \cdot \mathcal{F}(v^{(l)}) (k) \big)(x) + b^{(l)} \Big) \quad \forall x \in \Omega.
\end{equation}
where $k$ is the frequency mode index and $\mathcal{R}_\phi^{l}(k)$ is a complex-valued learnable weight matrix for the Fourier mode $k$ that approximates the kernel $\kappa_{\phi}$ in the frequency domain.

\subsection{Stochastic Adaptive Fourier Decomposition}
Let $\mathbb{D}=\{z\in\mathbb{C}:|z|<1\}$ denote the open unit disk in the complex plane, and let $\mathbb{T}=\partial\mathbb{D}$ denote its boundary. 
The Hardy space $H^2(\mathbb{D})$ is the Hilbert space of 
holomorphic functions on  $\mathbb{D}$ of the form
$f(z)=\sum_{m\ge 0} c_m z^m$ whose Taylor coefficients satisfy 
$\sum_{m\ge 0}|c_m|^2<\infty$. It is equipped with the inner product
\begin{equation}
    \langle f,g\rangle_{H^2(\mathbb{D})}
    =
    \frac{1}{2\pi}\int_0^{2\pi} f(e^{it})\overline{g(e^{it})}\,dt.
\end{equation}
For $\tilde f\in L^2(\mathbb{T})$ with Fourier expansion
$\tilde f(e^{it})=\sum_{m\in\mathbb{Z}} c_m e^{imt}$, its Hardy projection is defined by
\begin{equation}
    f^+(z):=(\mathcal P_+\tilde f)(z)=\sum_{m\ge 0} c_m z^m \in H^2(\mathbb{D}),
    \qquad z\in\mathbb{D}.
\end{equation}
If $\tilde f$ is real-valued, its analytic signal is defined by
\begin{equation}
    \mathcal A\tilde f(e^{it})=\tilde f(e^{it})+i\mathcal H\tilde f(e^{it}),
\end{equation}
where $\mathcal{H}: L^2(\mathbb{T}) \rightarrow L^2(\mathbb{T})$ is the (periodic) Hilbert transform \citep{Intrinsic_mono_component_qian}. Clearly we have $\tilde f(e^{it})=\Re(\mathcal A\tilde f(e^{it}))$.

For $a\in\mathbb{D}$, called a pole, the normalized Szeg\H{o} kernel is
\begin{equation}
    e_a(z)=\frac{\sqrt{1-|a|^2}}{1-\overline a z},
    \qquad
    \langle f,e_a\rangle_{H^2(\mathbb{D})}
    =
    \sqrt{1-|a|^2}\,f(a).
\end{equation}
Given a pole sequence $\mathbf{a}=(a_1,\dots,a_K)\in\mathbb{D}^K$, the associated Takenaka-Malmquist (TM) system is
\begin{equation}
    B_1(z; \mathbf{a})=e_{a_1}(z),
    \qquad
    B_k(z; \mathbf{a})
    =
    e_{a_k}(z)\prod_{m=1}^{k-1}\frac{z-a_m}{1-\overline{a_m}z},
    \quad k\ge 2. \label{eq:1d_tm_system}
\end{equation}

Classical AFD selects poles recursively according to the maximal selection principle \citep{AFD_Qian,2DAFD_Qian} starting from $f_1$, the larger is $\big|\langle f_k,e_a\rangle_{H^2(\mathbb D)}\big|^2$, the smaller is the energy of the reduced remainder $f_{k+1}$,
\begin{equation}
    a_k\in
    \arg\max_{a\in\mathbb D}
    \big|\langle f_k,e_a\rangle_{H^2(\mathbb D)}\big|^2
    =
    \arg\max_{a\in\mathbb D}
    (1-|a|^2)|f_k(a)|^2.
\end{equation}
Let $\omega \in \Omega$ denote a signal or sample; the stochastic AFD instead maximizes the expected energy across multiple signals
\begin{equation}
    a_k\in
    \arg\max_{a\in\mathbb D}
    \mathbb E_\omega\!\left[
        \big|\langle f_k(\cdot,\omega),e_a\rangle_{H^2(\mathbb D)}\big|^2
    \right].
\end{equation}

\section{Stochastic Adaptive Fourier Decomposition for Operator Learning}\label{sec:safdno}
Building on the SAFD framework, this section introduces the proposed Stochastic Adaptive Fourier Decomposition Neural Operator (SAFDNO). SAFDNO can be viewed as an 
adaptive neural operator that replaces the fixed Fourier basis used in FNO \citep{fno} with a data-dependent TM 
system. Specifically, instead of using a fixed Fourier basis with hard frequency truncation, the model first predicts poles adaptively in the unit disk, and then constructs the corresponding TM 
system according to adaptive pole placements. The key distinction from FNO is that the spectral components in SAFDNO are not fixed a priori but are adaptively inferred from the data, and the pole placements determine the localization strength of the resulting TM representations, enabling a more flexible and data-aligned representation.


\subsection{Neural Pole Predictor}
As discussed earlier, the classic stochastic AFD selects poles sequentially with a greedy procedure \citep{AFD_Qian,2DAFD_Qian}, which chooses the pole that captures the largest amount of expected energy at each step. This recursive pole selection process is computationally expensive for operator learning. To address this issue, 
SAFDNO replaces this greedy search with a neural pole predictor that amortizes pole selection across samples. For a problem posed on a domain \(\Omega\subset\mathbb{R}^d\), given the hidden representation \(v\), we predict a collection of pole sequences shared across hidden channels
\begin{equation}
    \mathbf{a}=\bigl(a^{(1)},a^{(2)},\dots,a^{(d)}\bigr) = \Psi_\theta(v),
    \label{eq:pole_predictor}
\end{equation}
where \(\Psi_\theta\) is the amortized neural pole predictor, and,
for each dimension \(r=1,\dots,d\), the corresponding pole sequence is
\[
\mathbf{a}^{(r)}=\bigl(a^{(r)}_1,\dots,a^{(r)}_K\bigr)\in\mathbb{D}^K.
\]
\begin{wrapfigure}{r}{0.38\textwidth}
  \centering
  \includegraphics[width=0.38\textwidth]{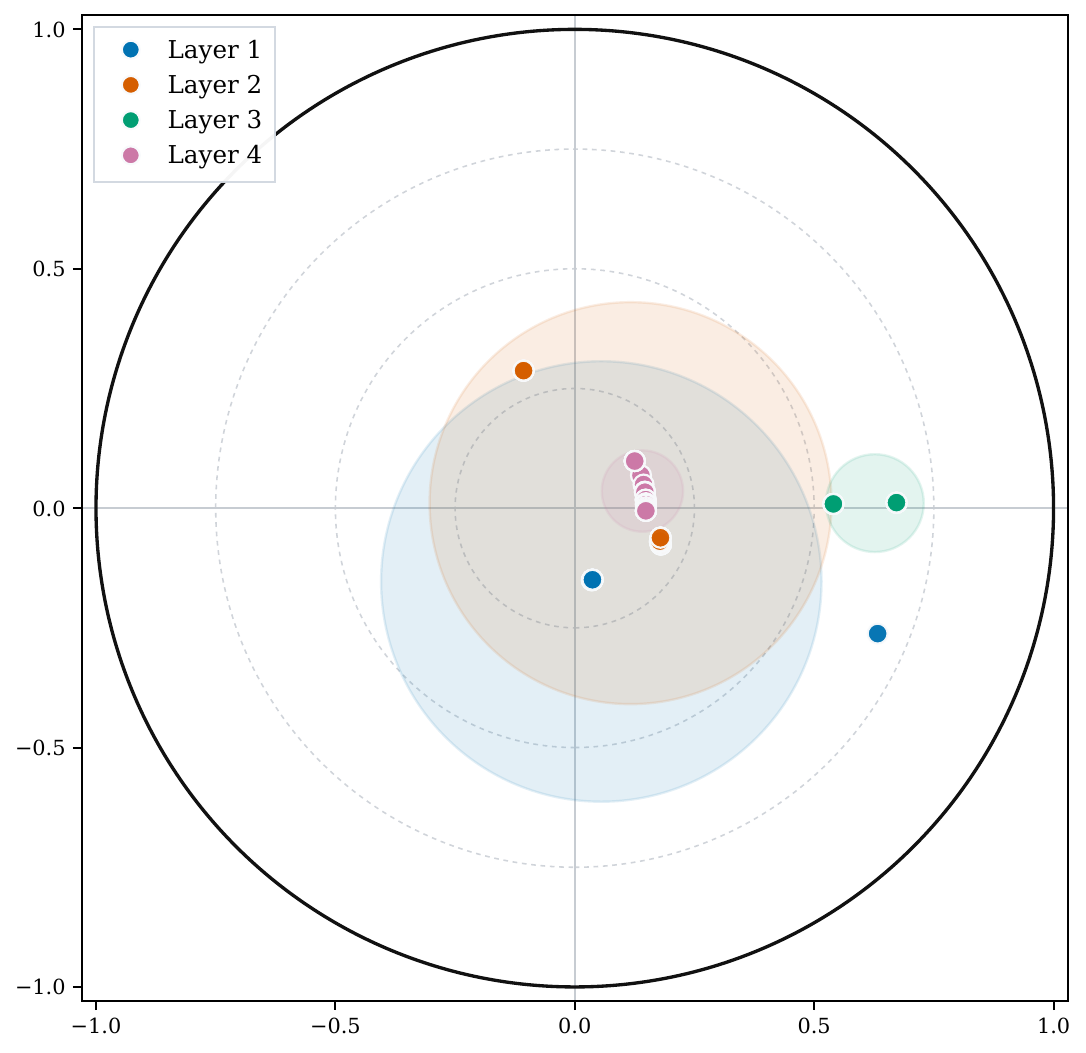}
  \caption{Learned poles on the Darcy flow. Representations for deeper layers are increasingly global, while earlier layers focus on localized and sharp structures.}
  \label{fig:pole_darcy}
\end{wrapfigure}
Specifically, we first map the hidden representation $v$ into a set of latent tokens $X=[x_1,\dots,x_T]$, where each token $x_t$ is formed from the local feature $v_t$ together with its positional embedding. Let $\{\xi_t\}_{t=1}^T \subset [0, 1]^d$ denote the normalized coordinates of the tokens, and $\gamma(\xi_t)$ be the corresponding Fourier positional embedding. Then, each token embedding $x_t$ is obtained by linear projection such that
\begin{equation}
    x_t = W[v_t \| \gamma(\xi_t)] + b,
\end{equation}
where $\|$ denotes the concatenation operation. Let $S \in \mathbb{R}^{dK \times d_h}$ be the learnable pole slots, and we implement \(\Psi_\theta\) as a query-based cross-attention module that is inherently robust to varying discretizations, in which the $dK$ learnable pole slots serve as queries that aggregate information from the $T$ input tokens to produce $dK$ adaptive poles, controlling the number of poles independent of arbitrary discretizations $T$. For each block \(\ell=1,\dots,L\) and an initial $S^{(0)}$, we update the pole slots by
\begin{align*}
\mathrm{head}^{(\ell)}_h
&=
\mathrm{Attention}\big(
S^{(\ell-1)}W_h^{Q,(\ell)},
XW_h^{K,(\ell)},
XW_h^{V,(\ell)}
\big).\\
U^{(\ell)}
&=
[\mathrm{head}^{(\ell)}_1 \| \dots \|\mathrm{head}^{(\ell)}_H]W^{O,(\ell)},\\
\widehat S^{(\ell)} &= S^{(\ell-1)} + U^{(\ell)},\\
S^{(\ell)} &= \widehat S^{(\ell)} + \mathrm{MLP}(\widehat S^{(\ell)}).
\end{align*}
After $L$ steps, the final slot states $S^{(L)}$ are mapped to the raw pole parameters through a slot-wise MLP $g_\theta$ such that
\begin{equation}
    P = g_\theta (S^{(L)})\in\mathbb{R}^{dK\times 2},
\end{equation}
where $P= [P^{(1)};P^{(2)};\dots;P^{(d)}]$ and $P^{(r)} = [(\widehat{\rho}^{(r)}_k,\widehat{\phi}^{(r)}_k)^K_{k=1}]\in\mathbb{R}^{K\times 2}$.  
Finally, we construct each pole from the output $(\widehat{\rho}^{(r)}_k,\widehat{\phi}^{(r)}_k)$ by the sigmoid and tanh activations and scaling as follows
\begin{equation}
\rho^{(r)}_k = (1-\varepsilon)\sigma (\widehat{\rho}^{(r)}_k),
\quad
\phi^{(r)}_k = \pi \tanh (\widehat{\phi}^{(r)}_k),
\quad
a^{(r)}_k = \rho^{(r)}_k e^{\mathrm{i}\phi^{(r)}_k},
\end{equation}
where 
$\varepsilon>0$ is a small constant ensuring each \(a^{(r)}_k\in\mathbb{D}\) for $k=1, 2, ..., K$ and $r=1, 2, ..., d$. $\rho_k^{(r)}$ determines the pole radius, while $\phi_k^{(r)}$ determines its angular position on the unit disk.

\subsection{Takenaka-Malmquist System}
The Takenaka-Malmquist (TM) system is an orthonormal system of rational functions in Hardy space $H^2(\mathbb{T})$ and is constructed using the pole sequence $\{ a_k \}_{k=1}^K$ within the unit disk; see Eq.~\eqref{eq:1d_tm_system}. As adaptive poles parameterize the TM system and control its localization, the resulting TM system is therefore data-adaptive and can adapt to non-stationary and localized patterns efficiently.

\begin{remark}[Relation between TM system and Fourier basis]
\label{rem:tm_fourier_1d}
If all poles are placed at the origin $\mathbf{a}=(0, 0, ..., 0) = \mathbf{0}$, that is, $a_k\equiv 0$ for $k=1,\dots,K$, then
\[
    B_1(z;\mathbf{0})=1,
    \qquad
    B_k(z;\mathbf{0})=z^{k-1},\qquad 2\le k\le K.
\]
Hence, the one-dimensional TM system Eq.~\eqref{eq:1d_tm_system} reduces to the nonnegative frequency Fourier basis of $H^2(\mathbb{T})$ \citep{AFD_Qian}. Moreover, for a nonzero pole $a=\rho e^{i\phi}$,
\[
    |e_a(e^{it})|^2=\left|\frac{\sqrt{1-|a|^2}}{1-\overline{a}e^{it}}\right|^2 = \frac{1-\rho^2}{1+\rho^2-2\rho\cos(t-\phi)},
\]
which is maximized at $t=\phi$ and becomes increasingly concentrated there as $\rho\to 1^-$. This explains why TM atoms can adapt to localized or nonstationary structures more efficiently than globally fixed Fourier modes.
\end{remark}

In our 
multi-dimensional case of $d>1$, $\mathbf{a}=(\mathbf{a}^{(1)},\mathbf{a}^{(2)},\dots,\mathbf{a}^{(d)})$ 
and each $\mathbf{a}^{(r)}=(a_1^{(r)},\dots,a_K^{(r)})\in \mathbb{D}^K$ ($r=1, 2, \dots, d$), the product-TM system with $k=(k_1,\dots,k_d)\in [K]^d$ and $z=(z_1,\dots,z_d) \in \mathbb{D}^d$ is formed by taking the tensor product of the one-dimensional TM systems,
\begin{equation}\label{eq:product_tm}
    B_k(z;\mathbf{a})=\prod_{r=1}^d B_{k_r}(z_r;\mathbf{a}^{(r)}).
\end{equation}

\begin{proposition}[Orthonormality of the product-TM system]
\label{prop:product_tm_orth}
The family $\{B_k(\cdot;\mathbf{a})\}_{k\in[K]^d}$ defined by Eq.~\eqref{eq:product_tm} is orthonormal in $H^2(\mathbb{T}^d)$. Equivalently, for any $k=(k_1,\dots,k_d)$ and $\ell=(\ell_1,\dots,\ell_d)$ in $[K]^d$,
\[
    \langle B_k(\cdot;\mathbf{a}),B_\ell(\cdot;\mathbf{a})\rangle_{H^2(\mathbb{T}^d)}
    =\prod_{r=1}^d \delta_{k_r\ell_r}.
\]
Consequently, the Bessel inequality holds:
\[
    \sum_{k\in[K]^d}\big|\langle v,B_k(\cdot;a)\rangle_{H^2(\mathbb{T}^d)}\big|^2
    \le \|v\|_{H^2(\mathbb{T}^d)}^2,
    \qquad \forall v\in H^2(\mathbb{T}^d).
\]
\end{proposition}
The proof of Proposition~\ref{prop:product_tm_orth} is given in Appendix~\ref{app:product_tm_orth_full_proof}

\subsection{Stochastic Adaptive Fourier Decomposition Operator}
Let $\sigma \in \Sigma_d$ index the analytic signal branches of the computational coordinate representation $v$ on the torus $\mathbb{T}^d$, where
\begin{equation}
    \Sigma_1=\{+\}, \quad \Sigma_2=\{++, +-\}, \quad \Sigma_3=\{+++, ++-, +-+, +--\}.
\end{equation}
For real-valued signals, the remaining branches are determined by conjugate symmetry. Then, we have $\mathcal{A}_\sigma:=\mathcal{A}_{\sigma_1}^{(1)}\cdots\mathcal{A}_{\sigma_d}^{(d)}$, and we use $v^\sigma=\mathcal{A}_\sigma v$ to denote the analytic component of $v$ for $\sigma \in \Sigma_d$ accordingly. Additionally, omitted frequency parts are recovered by conjugate symmetry, so averaging over $\Sigma_d$ and taking the real part returns real-valued signals.

For $\sigma\in\Sigma_d$ and $k\in[K]^d$, the adaptive SAFD spectral coefficient is defined as
\begin{equation}\label{eq:tm_coeff}
    C^\sigma_k = \langle v^\sigma , B_k (z; \mathbf{a})\rangle_{H^2},
\end{equation}
where the inner product is taken channel-wise. Thus, $C^\sigma_k \in \mathbb{C}^{d_v}$ is the complex expansion coefficient associated with the $k$-th adaptive TM mode on the analytic signal branch $\sigma$.
For these coefficients in the adaptive spectral domain, we define a neural network for filtering. Specifically, for each $k\in[K]^d$, we introduce a learnable linear transform
\begin{equation}\label{eq:afd_kernel}
    R_k(\mathbf{a};\theta):=\psi_\theta(k;\mathbf{a})\in\mathbb{C}^{d_v\times d_v},
\end{equation}
where $\psi_\theta$ is a neural network conditioned on the pole collection $a$. In this work, we use a sinusoidal representation network (SIREN) \citep{siren} as $\psi_\theta$. The corresponding SAFD kernel integral operator is then defined by
\begin{equation}\label{eq:afd_operator}
    (\mathcal{K}(\mathbf{a};\theta)v)(z)
    :=\Re\left(
        \frac{1}{|\Sigma_d|}
        \sum_{\sigma\in\Sigma_d}
        \sum_{k\in[K]^d}
        R_k(\mathbf{a};\theta)\,C_k^\sigma\,B_k(z;\mathbf{a})
    \right).
\end{equation}
Taking the real part of the combined branch reconstructions yields a real-valued representation on the computational torus $\mathbb{T}^d$. For fixed $\mathbf{a}$ and $\theta$, the resulting operator is well-defined and bounded on $L^2(\mathbb{T}^d;\mathbb{R}^{d_v})$, as stated below.
\begin{proposition}[Well-posedness of the SAFD kernel integral operator]
\label{prop:afd_kernel_bounded}
Fix a pole collection $\mathbf{a}$ and parameters $\theta$, and define
\[
    M(\mathbf{a},\theta):=\max_{k\in[K]^d}\|R_k(\mathbf{a};\theta)\|_{\mathrm{op}}.
\]
Then the operator $\mathcal{K}(\mathbf{a};\theta):L^2(\mathbb{T}^d;\mathbb{R}^{d_v})\to L^2(\mathbb{T}^d;\mathbb{R}^{d_v})$ defined by Eq.~\eqref{eq:afd_operator} is well-defined, $\mathbb{R}$-linear, and bounded. 
Moreover,
\[
    \|\mathcal{K}(\mathbf{a};\theta)v\|_{L^2}
    \le 2^d M(\mathbf{a},\theta)\,\|v\|_{L^2},
    \qquad
    \forall v\in L^2(\mathbb{T}^d;\mathbb{R}^{d_v}).
\]
\end{proposition}
The proof of Proposition~\ref{prop:afd_kernel_bounded} is given in Appendix~\ref{app:afd_kernel_bounded_full}.

Following \citep{amfno_xiao2024,ffno}, each SAFDNO layer is updated by
\begin{align}
    v^{(l+1)}(z)
    =v^{(l)}(z)
    +W_2^{(l)}\,\phi\!\left(
        W_1^{(l)}\bigl(\mathcal{K}(\mathbf{a}^{(l)};\theta_r)v^{(l)}\bigr)(z)+b_1^{(l)}
    \right)
    +b_2^{(l)}.
\end{align}
For structured-mesh PDEs, one may augment the representation
\begin{equation}
    v^{(l+1)}(z)
    =
    v^{(l)}(z)
    +
    W_2^{(l)}
    \phi\!\left(
        W_1^{(l)}
        \bigl[
        s^{(l)}(z),\{\delta_r s^{(l)}(z)\}_{r=1}^{d},\Delta_h s^{(l)}(z)
        \bigr]
        +
        b_1^{(l)}
    \right)
    +
    b_2^{(l)},
\end{equation}
where we denote $s^{(l)}(z):= (\mathcal{K}(\mathbf{a}^{(l)};\theta_r)v^{(l)})(z)$ for simplicity, $\delta_r s^{(l)}(z)$ denotes the discrete first-order central difference operator along the $r$-th spatial dimension, and $\Delta_h s^{(l)}(z)$ is the discrete Laplacian operator.
S

\section{Experiments}
This section presents a comprehensive empirical evaluation of the proposed SAFDNO against several neural operator baselines on a range of standard PDE benchmarks. 

\subsection{Implementation Details}\label{sec:implementation}
All experiments are conducted using the AdamW optimizer for 500 training epochs with a learning rate of $1e-3$. The weight decay is set to $1e-5$ for the experiments in Table~\ref{tab:main_results_1} and $1e-4$ for those in Table~\ref{tab:main_results_2}. For the learning rate schedule, we use the one-cycle policy for all experiments in Table~\ref{tab:main_results_1} and cosine annealing for all experiments in Table~\ref{tab:main_results_2}. All experiments are performed on a single NVIDIA RTX 5090 GPU. Detailed experiment setups are presented in Appendix~\ref{app:detaled_exp_setups}.

\paragraph{Benchmarks and Baselines.} We use several popular PDE benchmark datasets, including Darcy and Navier-Stokes \citep{fno,kossaifi2025librarylearningneuraloperators}; Pipe, Airfoil, and Plasticity in \citep{geofno}. In addition, we also employ the Burgers equation, 1D Computational Fluid Dynamics (CFD), 2D CFD, and Reaction-Diffusion from PDEBench \citep{takamoto2022pdebench}. As the proposed SAFDNO is a spectral neural operator, we mainly compare SAFDNO with DeepONet \citep{lu2021learning}, FNO \citep{fno}, MWT \citep{mwt}, U-FNO \citep{wen2022u}, Geo-FNO \citep{geofno}, U-NO \citep{rahman2023uno}, F-FNO \citep{ffno}, LSM \citep{wu2023LSM}, and AM-FNO \citep{amfno_xiao2024}. Additionally, we also consider several attention-based neural operators, including Galerkin \citep{cao2021choose-a57}, OFormer \citep{oformer}, GNOT \citep{gnot}, ONO \citep{xiao2024improved}, and Transolver \citep{wu2024Transolver}.


\paragraph{Evaluation Metric.} The $\ell_2$ relative error is used as a metric for training and evaluation. Given data pairs $\{a_i, u_i \}_{i=1}^N$ and a model $\mathcal{G}_\theta: \mathcal{A} \rightarrow \mathcal{U}$, the $\ell_2$ relative error is computed as
\begin{equation}
    L (\mathcal{G}_\theta(a), u) = \frac1N \sum_{i=1}^N \frac{\|\mathcal{G}_\theta(a_i) - u_i \|^2}{\| u_i \|^2}.
\end{equation}

\subsection{Main Results}\label{sec:main_results}
Tables~\ref{tab:main_results_1} and~\ref{tab:main_results_2} report the main results across nine PDE benchmarks. Overall, SAFDNO achieves the best performance on all six regular-grid benchmarks, demonstrating consistent improvements, especially on problems for which a fixed Fourier basis is known to be restrictive. On the Burgers equation, whose solutions develop sharp discontinuities requiring information from a wide spectrum of modes, SAFDNO achieves a relative $\ell_2$ error of $5.79e$-$3$, reducing the error of FNO by $48.3\%$. These results are consistent with our motivation and Remark~\ref{rem:tm_fourier_1d}, which suggest that the adaptive TM system can allocate representation capacity toward localized high-frequency structures without the information loss incurred by hard frequency truncation (see Figures~\ref{fig:multiscale_ns} and~\ref{fig:localized_darcy}).

\begin{table}[htbp]
\centering
\caption{Test relative $\ell_2$ error on five standard PDE benchmarks. $^\dagger$ denotes ablation variants. / indicates that FNO is not applicable to this task.}
\begin{tabular}{l|c c c c c}
\toprule
\multicolumn{1}{c|}{\textsc{Model}} & \multicolumn{3}{c}{\textsc{Structured Mesh}} & \multicolumn{2}{c}{\textsc{Regular Grid}} \\
\cmidrule(lr){2-4}\cmidrule(lr){5-6}
\multicolumn{1}{c|}{} & \textsc{Plasticity} & \textsc{Airfoil} & \textsc{Pipe} & \textsc{Navier-Stokes} & \textsc{Darcy} \\
\midrule
\textsc{DeepONet}    &  0.0079   & 0.0130 & 0.0097  & 0.4508 & 0.0411 \\
\textsc{FNO}         & /      & /      & /      & 0.1468 & 0.0108 \\
\textsc{MWT}         & 0.0076 & 0.0075 & 0.0077 & 0.1541 & 0.0082 \\
\textsc{U-FNO}       & 0.0039 & 0.0269 & 0.0056 & 0.2231 & 0.0183 \\
\textsc{Geo-FNO}     & 0.0074 & 0.0138 & 0.0067 & 0.1556 & 0.0108 \\
\textsc{U-NO}        & 0.0034 & 0.0078 & 0.0100 & 0.1713 & 0.0113 \\
\textsc{F-FNO}       & 0.0047 & 0.0078 & 0.0070 & 0.2322 & 0.0077 \\
\textsc{LSM}         & 0.0025 & 0.0059 & 0.0050 & 0.1535 & 0.0065 \\
\midrule
\textsc{Galerkin}    & 0.0120 & 0.0118 & 0.0098 & 0.1401 & 0.0084 \\
\textsc{OFormer}     & \underline{0.0017} & 0.0183 & 0.0168 & 0.1705 & 0.0124 \\
\textsc{GNOT}        & 0.0336 & 0.0076 & 0.0047 & 0.1380 & 0.0105 \\
\textsc{ONO}         & 0.0048 & 0.0061 & 0.0052 & 0.1195 & 0.0076 \\
\textsc{Transolver}   & \textbf{0.0012} & \textbf{0.0053} & \textbf{0.0033} & \underline{0.0900} & \underline{0.0057} \\
\midrule
\textbf{$^\dagger$ SAFD + Eq.~\eqref{eq:FNO_layer}} & / & / & / & 0.1344 & 0.0091\\
\textbf{$^\dagger$ SAFD + SIREN + Eq.~\eqref{eq:FNO_layer}} & / & / & / & 0.0962 & 0.0064\\
\textbf{SAFDNO (Ours)}       & 0.0039 & \underline{0.0057} & \underline{0.0047} & \textbf{0.0684} & \textbf{0.0054} \\
\bottomrule
\end{tabular}
\label{tab:main_results_1}
\end{table}

For Navier-Stokes and Darcy flow in Table~\ref{tab:main_results_1}, SAFDNO also outperforms all baselines, including attention-based neural operators, such as Transolver. This is particularly remarkable since Navier-Stokes solutions exhibit complex multiscale vortical structures that can benefit from the input-adaptive spectral representation provided by SAFDNO, as shown in Figure~\ref{fig:multiscale_ns}.

\begin{figure}[htbp]
    \centering
    \includegraphics[width=\linewidth]{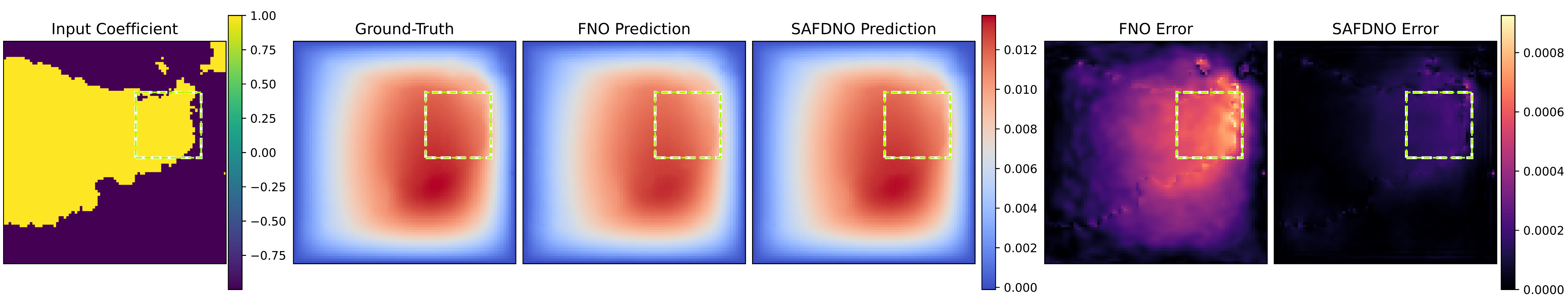}
    \caption{Localized pattern on Darcy flow. The highlighted region indicates a representative localized pattern where SAFDNO better matches the ground truth.}
    \label{fig:localized_darcy}
\end{figure}

For structured-mesh problems, attention-based operators such as Transolver remain generally superior to spectral neural operators, as attention-based operators do not assume grid regularity, whereas spectral methods generally require a regular grid for basis expansion. Nevertheless, our SAFDNO remains competitive overall, 
especially when compared with other spectral operator baselines, such as U-FNO, LSM, etc. The visualizations of the prediction performance are given in Appendix~\ref{app:performance_visual}. We provide additional results, including efficiency analysis between SAFDNO and FNO, in Appendix~\ref{app:additional_results}.
\begin{table}[!htbp]
\centering
\caption{Test relative $\ell_2$ error performance on additional regular-grid PDEs. $^\dagger$ denotes ablation variants.}
\begin{tabular}{l|c c c c}
\toprule
 & \textsc{Burgers} & \textsc{1D CFD} & \textsc{2D CFD} & \textsc{Reaction-Diffusion}\\
\midrule
\textsc{DeepONet} & $2.62e$-$1$ & $1.99e$-$2$ & $2.76e$-$2$ & $5.08e$-$4$\\
\textsc{FNO}      & $1.12e$-$2$ & $8.25e$-$3$ & $1.66e$-$2$ & $1.17e$-$4$\\
\textsc{U-FNO}    & $1.70e$-$2$ & $8.27e$-$3$ & $1.17e$-$2$ & $1.75e$-$4$\\
\textsc{F-FNO}    & $1.55e$-$2$ & $8.80e$-$3$ & $1.69e$-$2$ & $8.47e$-$5$ \\
\textsc{U-NO}     & \underline{$7.53e$-$3$} & $7.72e$-$3$ & $1.57e$-$2$ & $7.66e$-$5$\\
\textsc{AM-FNO (MLP)}   & $8.92e$-$3$ & $7.66e$-$3$ & $1.21e$-$2$ & $1.78e$-$3$\\
\midrule
\textbf{$^\dagger$ SAFD + Eq.~\eqref{eq:FNO_layer}} & $9.08e$-$3$ & $6.64e$-$3$ & $1.28e$-$2$ & $6.02e$-$5$\\
\textbf{$^\dagger$ SAFD + SIREN + Eq.~\eqref{eq:FNO_layer}} & $7.63e$-$3$ & \underline{$5.89e$-$3$} & \underline{$6.48e$-$3$} & \underline{$5.20e$-$5$}\\
\textbf{SAFDNO (Ours)}    & \textbf{5.79e-3} & \textbf{5.72e-3} & \textbf{6.44e-3} & 
\textbf{4.59e-5}\\
\bottomrule
\end{tabular}
\label{tab:main_results_2}
\end{table}

\subsection{Ablation Results}
For the ablation study, \texttt{$^\dagger$ SAFD + Eq.~\eqref{eq:FNO_layer}} merely replaces the Fourier transform of FNO in Eq.~\eqref{eq:FNO_layer} with the SAFD driven by the neural pole predictor, while retaining the static learnable weights $R_\phi$ of the standard FNO and the plain FNO layer architecture. \texttt{$^\dagger$ SAFD + SIREN + Eq.~\eqref{eq:FNO_layer}} further replaces $R_\phi$ with the SIREN-based dynamic kernel $R_k(\mathbf{a}; \theta)$ of Eq.~\eqref{eq:afd_operator}, still using the FNO operator layer in Eq.~\eqref{eq:FNO_layer} and the SAFD components. As shown in Tables~\ref{tab:main_results_1} and~\ref{tab:main_results_2}, each added component brings a progressive improvement. The gain from \texttt{$^\dagger$ SAFD + Eq.~\eqref{eq:FNO_layer}} over FNO isolates the effect of the adaptive TM system itself, indicating that replacing the fixed Fourier basis with an input-adaptive TM system enhances the spectral representation and improves operator learning. The gain from \texttt{$^\dagger$ SAFD + SIREN + Eq.~\eqref{eq:FNO_layer}} shows that conditioning the spectral filter on the predicted poles, making the filter itself input-adaptive, provides further improvements.

\subsection{Zero-Shot Super-Resolution}
\begin{table}[htbp]
\centering
\caption{Zero-Shot Super-Resolution on 2D Darcy Flow. Operators are trained at $32 \times 32$ resolution and directly evaluated at all resolutions. FNO uses $12\times12$ retained Fourier modes, while SAFDNO has $12$ poles along each spatial dimension, hence $24$ poles in total.}
\setlength{\tabcolsep}{10pt}
\begin{tabular}{l|c c c c}
\toprule
\multicolumn{1}{c|}{} & Params & $32 \times 32$ & $64 \times 64$ & $128 \times 128$\\
\midrule
\textsc{FNO} & $701,281$ & $0.0730$ & $0.0832$ & $0.1031$\\
\textbf{SAFDNO (Ours)}  & \textbf{447,769} & \textbf{0.0413} & \textbf{0.0528} & \textbf{0.0643}\\
\bottomrule
\end{tabular}
\label{tab:zero_shot}
\end{table}
Table~\ref{tab:zero_shot} evaluates zero-shot super-resolution performance on 2D Darcy flow. SAFDNO outperforms FNO across all tested resolutions. In particular, when evaluated on finer grids, SAFDNO exhibits less performance degradation than FNO 
while using $36.1\%$ fewer parameters. This indicates that the TM representation, which encodes spectral information through pole locations, is inherently less sensitive to changes in 
discretization and provides a more compact spectral representation. The visualization of zero-shot super-resolution performance is provided in Figure~\ref{fig:afdno_zeroshot_darcy} in Appendix~\ref{app:performance_visual}.


\section{Conclusion}
We presented SAFDNO, an adaptive spectral neural operator that replaces fixed truncated Fourier modes with the input-dependent TM system and amortized pole prediction. By combining SAFD-inspired adaptive representations with neural operator learning, SAFDNO improves the modeling of localized, nonstationary, and fine-scale PDE solution structures while retaining the global receptive field benefits of spectral operators. Empirically, SAFDNO achieves strong accuracy across diverse regular-grid PDE benchmarks and more robust zero-shot super-resolution performance on Darcy flow. Overall, our results highlight adaptive rational bases as a promising direction for operator learning.

\bibliographystyle{unsrt}
\bibliography{references.bib}

\newpage
\appendix

\section*{Appendix}

\section{Extended Preliminaries}\label{app:extend_preliminaries}

\subsection{Operator Learning}
Operator learning aims to approximate maps between infinite-dimensional function spaces rather than between finite-dimensional vector spaces. Let $\Omega \subset \mathbb{R}^d$ be a bounded domain, $\mathcal{A}=\mathcal{A}(\Omega, \mathbb{R}^{d_a})$ be the input function space, and $\mathcal{U} = \mathcal{U}(\Omega, \mathbb{R}^{d_u})$ be the solution function space. Given training data pairs $\{(a_i, u_i) \}_{i=1}^N \subset \mathcal{A} \times \mathcal{U}$, we aim to train the neural operator
\begin{equation*}
    \mathcal{G}_\theta : \mathcal{A} \rightarrow \mathcal{U}
\end{equation*}
by minimizing the empirical risk
\begin{equation*}
    \min_{\theta \in \mathbb{R}^p} \frac1N \sum_{i=1}^N \ell \Big(\mathcal{G}_\theta (a_i), u_i \Big).
\end{equation*}

In general, the architecture of neural operators can be presented as
\begin{equation*}
    \mathcal{G}_\theta := \mathcal{P} \circ \mathcal{L}^L \circ \cdots \circ \mathcal{L}^1 \circ \mathcal{Q},
\end{equation*}
where $\mathcal{Q}: \mathcal{A} (\Omega, \mathbb{R}^{d_a}) \rightarrow \mathcal{V}(\Omega, \mathbb{R}^{d_v})$ is a lifting layer, $\mathcal{L}^l: \mathcal{V}(\Omega, \mathbb{R}^{d_v}) \rightarrow \mathcal{V}(\Omega, \mathbb{R}^{d_v})$ are operator layers acting on the hidden representations, and $\mathcal{P}: \mathcal{V}(\Omega, \mathbb{R}^{d_v}) \rightarrow \mathcal{U}(\Omega, \mathbb{R}^{d_u})$ is a projection layer.

A general neural operator layer has the form
\begin{equation*}
    v^{(l+1)} (x) = \sigma \Big( W^{(l)} v^{(l)} (x) + \int_\Omega \kappa_\phi^{(l)} \big(x, y, a(x), a(y)\big) v^{(l)} (y) dy + b^{(l)} \Big), \quad \forall x \in \Omega,
\end{equation*}
where $W^{(l)}: \mathbb{R}^{d_v} \rightarrow \mathbb{R}^{d_v}$ is a point-wise linear map, $\kappa_\phi^{(l)}$ is a learnable kernel, $b^{(l)}$ is a bias term, and $\sigma$ is the activation function.

\subsection{Fourier Neural Operator}
Fourier Neural Operator (FNO) \citep{fno} uses global convolution in the Fourier domain, and its kernel integral operator satisfies translation invariance such that
\begin{equation*}
    \kappa_\phi (x, y, a(x), a(y)) =\kappa_\phi (x - y).
\end{equation*}
Under this assumption, the kernel operator becomes a convolution:
\begin{equation*}
    (K_{\phi}^{(\ell)}v)(x)
    =
    \int_{\Omega}\kappa_{\phi}^{(\ell)}(x-y)v(y)\,dy
    =
    \mathcal{F}^{-1}\Big(
        \mathcal{F}(\kappa_{\phi}^{(\ell)})(k)\cdot \mathcal{F}(v)(k)
    \Big)(x),
    \quad \forall x\in \Omega.
\end{equation*}
Accordingly, each Fourier layer is written as
\begin{equation*}
    v^{(\ell+1)}(x)
    =
    \sigma\Big(
        W^{(\ell)}v^{(\ell)}(x)
        +
        \mathcal{F}^{-1}\big(
            \mathcal{R}_{\phi}^{\ell}(k)\cdot \mathcal{F}(v^{(\ell)})(k)
        \big)(x)
        +
        b^{(\ell)}
    \Big),
    \quad \forall x\in \Omega,
\end{equation*}
where $k$ is the frequency index and $\mathcal{R}_{\phi}^{\ell}(k) \in \mathbb{C}^{d_v \times d_v}$ is the learnable weight matrix for each frequency.

In practice, FNO adopts a finite-dimensional parameterization by truncating the Fourier series at a maximal number of Fourier modes  $k_{\max} = \big| \{ k \in \mathbb{Z}^d: |k_j| \leq k_{\max, j}, j=1,\dots,d \} \big|$, and higher-frequency modes are discarded by padding zeros.

\subsection{Adaptive Fourier Decomposition}

Let $\mathbb{D}=\{z\in \mathbb{C}: |z|<1\}$ be the open unit disk in the complex plane, and let \(\mathbb{T}=\partial \mathbb{D}\) denote its boundary. The Hardy space \(H^2(\mathbb{D})\) consists of all holomorphic functions on \(\mathbb{D}\) of the form
\[
f(z)=\sum_{m\ge 0} c_m z^m,
\qquad
\sum_{m\ge 0}|c_m|^2<\infty,
\]
equipped with the inner product
\[
\langle f,g\rangle_{H^2(\mathbb{D})}
=
\frac{1}{2\pi}\int_{0}^{2\pi} f(e^{it})\overline{g(e^{it})}\,dt.
\]
For \(\widetilde{f}\in L^2(\mathbb{T})\) with Fourier expansion
\[
\widetilde{f}(e^{it})=\sum_{m\in\mathbb{Z}} c_m e^{imt},
\]
its Hardy projection is defined by
\[
f^{+}(z):=(P^{+}\widetilde{f})(z)=\sum_{m\ge 0} c_m z^m \in H^2(\mathbb{D}),
\qquad z\in \mathbb{D}.
\]
If \(\widetilde{f}\) is real-valued, its analytic signal is defined by
\[
A\widetilde{f}(e^{it})=\widetilde{f}(e^{it})+iH\widetilde{f}(e^{it}),
\]
where \(H:L^2(\mathbb{T})\to L^2(\mathbb{T})\) is the periodic Hilbert transform. In this case,
\[
\widetilde{f}(e^{it})=\Re\big(A\widetilde{f}(e^{it})\big).
\]

For \(a\in \mathbb{D}\), called a pole, the normalized Szeg\H{o} kernel is
\[
e_a(z)=\frac{\sqrt{1-|a|^2}}{1-\overline{a}z},
\qquad
\langle f,e_a\rangle_{H^2(\mathbb{D})}
=
\sqrt{1-|a|^2}\,f(a).
\]
Given a pole sequence \(a=(a_1,\ldots,a_K)\in \mathbb{D}^K\), the associated Takenaka-Malmquist (TM) system is
\[
B_1(z;a)=e_{a_1}(z),
\qquad
B_k(z;a)
=
e_{a_k}(z)\prod_{m=1}^{k-1}\frac{z-a_m}{1-\overline{a_m}z},
\qquad k\ge 2.
\]

For the analytic component \(f=f^{+}\in H^2(\mathbb{D})\), set \(f_1=f\). The AFD recursion is
\[
f_k(z)
=
\langle f_k,e_{a_k}\rangle_{H^2(\mathbb{D})}\,e_{a_k}(z)
+
f_{k+1}(z)\frac{z-a_k}{1-\overline{a_k}z},
\]
where the reduced remainder is
\[
f_{k+1}(z)
=
\left(
f_k(z)-\langle f_k,e_{a_k}\rangle_{H^2(\mathbb{D})}\,e_{a_k}(z)
\right)\frac{1-\overline{a_k}z}{z-a_k}.
\]
Since
\[
f_k(a_k)-\langle f_k,e_{a_k}\rangle_{H^2(\mathbb{D})}\,e_{a_k}(a_k)=0,
\]
the reduced remainder \(f_{k+1}\) again belongs to \(H^2(\mathbb{D})\). The pole \(a_k\) is selected by the maximal selection principle
\[
a_k
\in
\operatorname*{arg\,max}_{a\in \mathbb{D}}
\big|\langle f_k,e_a\rangle_{H^2(\mathbb{D})}\big|^2
=
\operatorname*{arg\,max}_{a\in \mathbb{D}}
(1-|a|^2)|f_k(a)|^2.
\]
After \(n\) steps, one has
\[
f(z)
=
\sum_{k=1}^{n}
\langle f_k,e_{a_k}\rangle_{H^2(\mathbb{D})}\,B_k(z;a)
+
f_{n+1}(z)\prod_{k=1}^{n}\frac{z-a_k}{1-\overline{a_k}z}.
\]
By orthogonality,
\[
\|f\|_{H^2(\mathbb{D})}^2
=
\sum_{k=1}^{n}
\big|\langle f_k,e_{a_k}\rangle_{H^2(\mathbb{D})}\big|^2
+
\|f_{n+1}\|_{H^2(\mathbb{D})}^2.
\]
Under consecutive maximal selections,
\[
f(z)
=
\sum_{k=1}^{\infty}
\langle f_k,e_{a_k}\rangle_{H^2(\mathbb{D})}\,B_k(z;a)
\qquad \text{in } H^2(\mathbb{D}).
\]

The orthogonal standard remainder is defined by
\[
g_k(z)
=
f(z)-\sum_{\ell=1}^{k-1}\langle f,B_\ell\rangle_{H^2(\mathbb{D})}\,B_\ell(z),
\qquad k\ge 2.
\]
Then
\[
\langle f_k,e_{a_k}\rangle_{H^2(\mathbb{D})}
=
\langle g_k,B_k\rangle_{H^2(\mathbb{D})}
=
\langle f,B_k\rangle_{H^2(\mathbb{D})},
\]
and
\[
g_k(z)
=
f_k(z)\prod_{\ell=1}^{k-1}\frac{z-a_\ell}{1-\overline{a_\ell}z}.
\]

\subsection{Stochastic Adaptive Fourier Decomposition}

Let \(\widetilde{f}(e^{it},\omega)\in L^2(\mathbb{T})\) be a real-valued family of signals, and let
\[
f(z,\omega):=(P^{+}\widetilde{f}(\cdot,\omega))(z)\in H^2(\mathbb{D}).
\]
Assume
\[
\mathbb{E}_{\omega}\|f(\cdot,\omega)\|_{H^2(\mathbb{D})}^2<\infty.
\]
Set
\[
f_1(z,\omega)=f(z,\omega).
\]
For each fixed \(\omega\),
\[
f_k(z,\omega)
=
\langle f_k(\cdot,\omega),e_{a_k}\rangle_{H^2(\mathbb{D})}\,e_{a_k}(z)
+
f_{k+1}(z,\omega)\frac{z-a_k}{1-\overline{a_k}z},
\]
where
\[
f_{k+1}(z,\omega)
=
\left(
f_k(z,\omega)
-
\langle f_k(\cdot,\omega),e_{a_k}\rangle_{H^2(\mathbb{D})}\,e_{a_k}(z)
\right)\frac{1-\overline{a_k}z}{z-a_k}.
\]
The pole \(a_k\) is selected by the stochastic maximal selection principle
\[
a_k
\in
\operatorname*{arg\,max}_{a\in \mathbb{D}}
\mathbb{E}_{\omega}
\left[
\big|
\langle f_k(\cdot,\omega),e_a\rangle_{H^2(\mathbb{D})}
\big|^2
\right].
\]
Accordingly, for each fixed \(\omega\),
\[
f(z,\omega)
=
\sum_{k=1}^{n}
\langle f_k(\cdot,\omega),e_{a_k}\rangle_{H^2(\mathbb{D})}\,B_k(z;a)
+
f_{n+1}(z,\omega)\prod_{k=1}^{n}\frac{z-a_k}{1-\overline{a_k}z}.
\]

Let
\[
g_k(z,\omega)
=
f(z,\omega)
-
\sum_{\ell=1}^{k-1}
\langle f(\cdot,\omega),B_\ell\rangle_{H^2(\mathbb{D})}\,B_\ell(z),
\qquad k\ge 2,
\]
be the \(k\)-th standard remainder. Then
\[
\langle f_k(\cdot,\omega),e_{a_k}\rangle_{H^2(\mathbb{D})}
=
\langle g_k(\cdot,\omega),B_k\rangle_{H^2(\mathbb{D})}
=
\langle f(\cdot,\omega),B_k\rangle_{H^2(\mathbb{D})}.
\]
Hence,
\[
\mathbb{E}_{\omega}
\left[
\big|
\langle f_k(\cdot,\omega),e_{a_k}\rangle_{H^2(\mathbb{D})}
\big|^2
\right]
=
\mathbb{E}_{\omega}
\left[
\big|
\langle g_k(\cdot,\omega),B_k\rangle_{H^2(\mathbb{D})}
\big|^2
\right]
=
\mathbb{E}_{\omega}
\left[
\big|
\langle f(\cdot,\omega),B_k\rangle_{H^2(\mathbb{D})}
\big|^2
\right].
\]
If
\[
B_k^a(z)
=
e_a(z)\prod_{\ell=1}^{k-1}\frac{z-a_\ell}{1-\overline{a_\ell}z},
\]
then the stochastic maximal selection principle can be equivalently written as
\[
a_k
\in
\operatorname*{arg\,max}_{a\in \mathbb{D}}
\mathbb{E}_{\omega}
\left[
\big|
\langle f(\cdot,\omega),B_k^a\rangle_{H^2(\mathbb{D})}
\big|^2
\right].
\]

\section{Stochastic Adaptive Fourier Decomposition Neural Operator}
\label{app:safdno_detail}

This appendix provides further details for the SAFDNO formulation in Section~\ref{sec:safdno}.

\subsection{Neural Pole Predictor}
\label{app:pole_predictor}

The classic stochastic AFD selects poles sequentially with a greedy procedure according to the maximal selection principle, choosing at each step the pole that captures the largest amount of expected energy. This recursive pole selection is computationally expensive for operator learning. SAFDNO replaces this greedy search with a neural pole predictor that amortizes pole selection across samples. For a problem posed on a domain $\Omega\subset\mathbb{R}^d$, given the hidden representation $v$, SAFDNO predicts a collection of pole sequences
\begin{equation}
    a
    =
    \bigl(a^{(1)},a^{(2)},\dots,a^{(d)}\bigr)
    =
    \Psi_\theta(v),
    \label{eq:poles_app}
\end{equation}
where $\Psi_\theta$ is the amortized neural pole predictor. For each spatial dimension $r=1,\dots,d$, the corresponding pole sequence is
\begin{equation}
    a^{(r)}
    =
    \bigl(a^{(r)}_1,\dots,a^{(r)}_K\bigr)
    \in \mathbb{D}^K .
\end{equation}
Since $\Psi_\theta$ is evaluated on the current hidden representation $v$, the predicted pole collection is input-dependent.

\paragraph{Token construction.}
Specifically, $\Psi_\theta$ is implemented as a query-based cross-attention module to directly infer poles from the input data. We first map the hidden representation $v$ into a set of latent tokens $X=[x_1,\dots,x_T]$, where each token $x_t$ is formed from the local feature $v_t$ together with its positional embedding. Let $\{\xi_t\}_{t=1}^T\subset[0,1]^d$ denote the normalized coordinates of the tokens, and let $\gamma(\xi_t)$ be the corresponding Fourier positional embedding. Then each token embedding is obtained by linear projection:
\begin{equation}
    x_t
    =
    W\bigl[v_t \,\|\, \gamma(\xi_t)\bigr]+b,
    \label{eq:token_app}
\end{equation}
where $\|$ denotes concatenation.

\paragraph{Pole slots and cross-attention updates.}
Let
\begin{equation}
    S^{(0)}
    =
    \bigl[S^{(0)}_1;S^{(0)}_2;\dots;S^{(0)}_d\bigr]
    \in \mathbb{R}^{dK\times d_h},
    \qquad
    S^{(0)}_r\in\mathbb{R}^{K\times d_h},
\end{equation}
denote the learnable pole slots, where $S^{(0)}_r$ is the bank of $K$ query slots associated with the $r$-th spatial dimension. For each cross-attention block $\ell=1,\dots,L$, the pole slots are updated by
\begin{align}
    \mathrm{head}^{(\ell)}_h
    &=
    \mathrm{Attention}
    \left(
        S^{(\ell-1)}W^{Q,(\ell)}_h,\;
        XW^{K,(\ell)}_h,\;
        XW^{V,(\ell)}_h
    \right),
    \\
    U^{(\ell)}
    &=
    \bigl[
        \mathrm{head}^{(\ell)}_1
        \,\|\,
        \cdots
        \,\|\,
        \mathrm{head}^{(\ell)}_H
    \bigr]W^{O,(\ell)},
    \\
    \widehat{S}^{(\ell)}
    &=
    S^{(\ell-1)}+U^{(\ell)},
    \\
    S^{(\ell)}
    &=
    \widehat{S}^{(\ell)}
    +
    \mathrm{MLP}\!\left(\widehat{S}^{(\ell)}\right).
\end{align}
After $L$ steps, the final slot states $S^{(L)}$ are mapped to raw pole parameters through a slot-wise MLP $g_\theta$:
\begin{equation}
    P
    =
    g_\theta\!\left(S^{(L)}\right)
    \in \mathbb{R}^{dK\times 2},
    \qquad
    P
    =
    \bigl[P^{(1)};P^{(2)};\dots;P^{(d)}\bigr],
\end{equation}
where
\begin{equation}
    P^{(r)}
    =
    \bigl[
        (\widehat{\rho}^{(r)}_k,\widehat{\phi}^{(r)}_k)
    \bigr]_{k=1}^K
    \in \mathbb{R}^{K\times 2}.
\end{equation}
Finally, each pole is constructed from the output $(\widehat{\rho}^{(r)}_k,\widehat{\phi}^{(r)}_k)$ by the sigmoid and tanh activations and scaling:
\begin{equation}
    \rho^{(r)}_k
    =
    (1-\varepsilon)
    \operatorname{sigmoid}\!\left(\widehat{\rho}^{(r)}_k\right),
    \qquad
    \phi^{(r)}_k
    =
    \pi\tanh\!\left(\widehat{\phi}^{(r)}_k\right),
    \qquad
    a^{(r)}_k
    =
    \rho^{(r)}_k e^{\mathrm{i}\phi^{(r)}_k}.
    \label{eq:pole_param_app}
\end{equation}
Here $\varepsilon>0$ is a small constant ensuring $a^{(r)}_k\in\mathbb{D}$ for every $k=1,\dots,K$ and $r=1,\dots,d$.

The pole collection $a=\Psi_\theta(v)$ specializes to
\begin{align}
    d=1:\quad
    &
    a
    =
    a^{(1)}
    =
    \bigl(a^{(1)}_1,\dots,a^{(1)}_K\bigr)
    \in \mathbb{D}^K,
    \\
    d=2:\quad
    &
    a
    =
    \bigl(a^{(1)},a^{(2)}\bigr),
    \qquad
    a^{(r)}\in\mathbb{D}^K,\quad r=1,2,
    \\
    d=3:\quad
    &
    a
    =
    \bigl(a^{(1)},a^{(2)},a^{(3)}\bigr),
    \qquad
    a^{(r)}\in\mathbb{D}^K,\quad r=1,2,3.
\end{align}
When the poles are placed at the origin, the one-dimensional TM system reduces to the nonnegative frequency Fourier basis of $H^2(\mathbb{T})$. As the magnitude of a pole approaches the unit circle, the corresponding Szeg\H{o} factor becomes increasingly concentrated near its phase location. This is consistent with Remark~\ref{rem:tm_fourier_1d} and its justification in Appendix~\ref{app:remark1_full}.

\paragraph{Pole distributions across layers.}
Figure~\ref{fig:pole_darcy_app} shows the learned poles of SAFDNO on the Darcy flow problem. In this example, the first two layers place poles over a relatively broad region of the unit disk. Layer 3 forms a compact group along the positive real direction and away from the origin, whereas layer 4 forms the most compact group near the origin. Thus, the learned poles do not simply move radially inward with depth; rather, the pole predictor changes both the spread and the location of the adaptive rational parameters across layers.

This behavior can be interpreted through the role of poles in the TM system. For a pole $a=\rho e^{\mathrm{i}\phi}$, the angle $\phi$ determines where the boundary magnitude of the associated Szeg\H{o} factor is maximized, while the radius $\rho$ controls the strength of this concentration. Therefore, the pole sequence in layer 3 indicates that this layer repeatedly uses a narrow family of TM atoms with a similar phase and localization scale. In contrast, the near-origin cluster in layer 4 corresponds to a weaker rational deformation of the TM system, where the angular location of the poles has a smaller effect. This suggests that, after the earlier SAFDNO layers have transformed the latent representation, the latter layer can represent the remaining dominant components with a less diverse pole sequence.

\begin{figure}[htbp]
\centering
\includegraphics[width=0.5\linewidth]{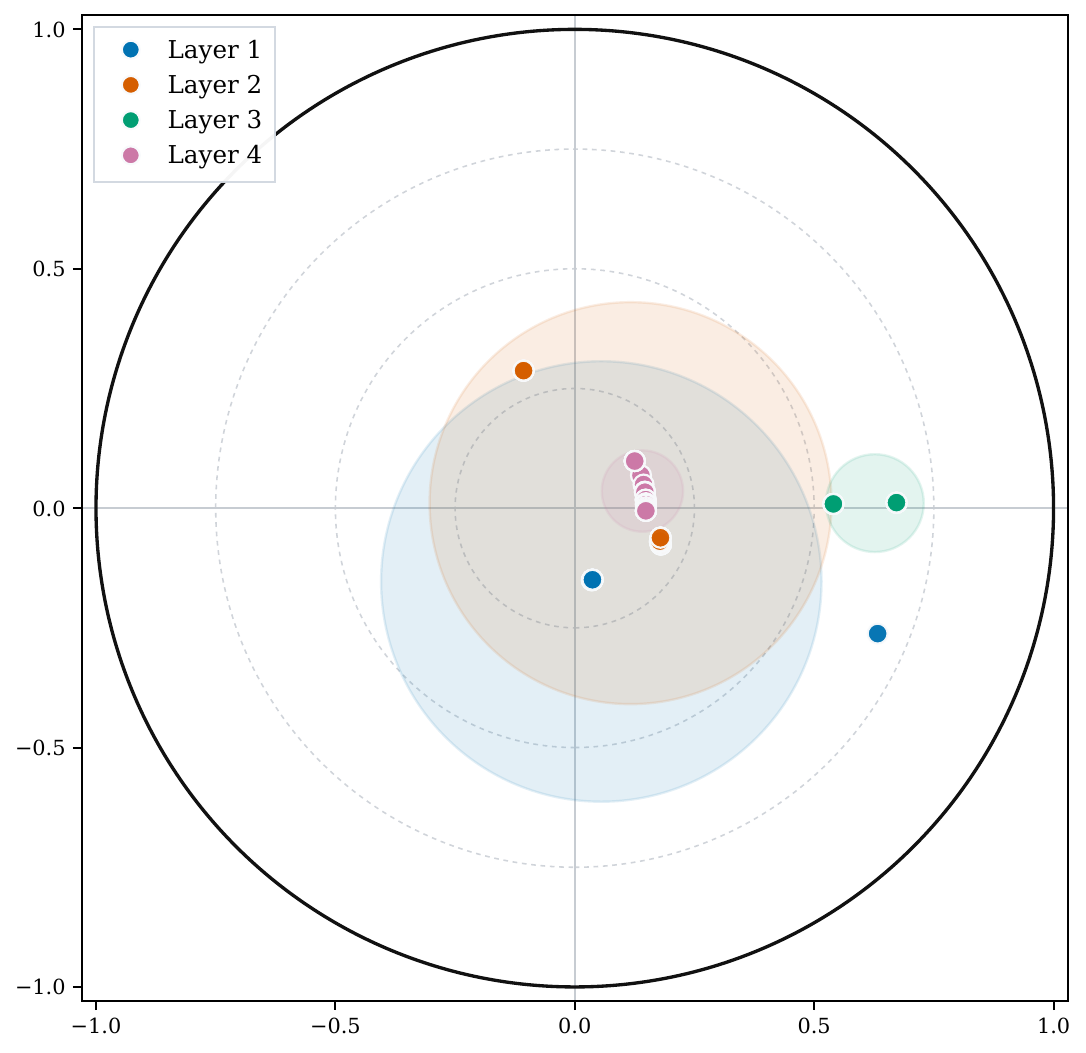}
\caption{Learned poles of SAFDNO on the Darcy flow problem across four SAFDNO layers. Each marker denotes a predicted pole in the unit disk $\mathbb{D}$.}
\label{fig:pole_darcy_app}
\end{figure}

\subsection{Takenaka-Malmquist System}
\label{app:tm_system}

Recall from Eq.~\eqref{eq:1d_tm_system} that, given a pole sequence $a=(a_1,\dots,a_K)\in\mathbb{D}^K$, the associated one-dimensional Takenaka-Malmquist (TM) system is
\begin{equation}
    B_1(z;a)
    =
    e_{a_1}(z),
    \qquad
    B_k(z;a)
    =
    e_{a_k}(z)
    \prod_{m=1}^{k-1}
    \frac{z-a_m}{1-\overline{a_m}z},
    \qquad
    k\ge 2,
    \label{eq:tm_1d_app}
\end{equation}
where
\begin{equation}
    e_a(z)
    =
    \frac{\sqrt{1-|a|^2}}{1-\overline{a}z}.
\end{equation}
For $d>1$ and
\[
    a
    =
    \bigl(a^{(1)},a^{(2)},\dots,a^{(d)}\bigr),
\]
the product-TM system is formed by taking the tensor product of the one-dimensional TM systems along each coordinate. For $k=(k_1,\dots,k_d)\in[K]^d$ and $z=(z_1,\dots,z_d)\in\mathbb{D}^d$,
\begin{equation}
    B_k(z;a)
    =
    \prod_{r=1}^d
    B_{k_r}\!\left(z_r;a^{(r)}\right).
    \label{eq:product_tm_app}
\end{equation}
The specialization to $d=1,2,3$ reads
\begin{align}
    d=1:\quad
    &
    B_k(z;a)
    =
    B_k\!\left(z_1;a^{(1)}\right),
    \\
    d=2:\quad
    &
    B_{(k_1,k_2)}(z;a)
    =
    B_{k_1}\!\left(z_1;a^{(1)}\right)
    B_{k_2}\!\left(z_2;a^{(2)}\right),
    \\
    d=3:\quad
    &
    B_{(k_1,k_2,k_3)}(z;a)
    =
    B_{k_1}\!\left(z_1;a^{(1)}\right)
    B_{k_2}\!\left(z_2;a^{(2)}\right)
    B_{k_3}\!\left(z_3;a^{(3)}\right).
\end{align}
By Proposition~\ref{prop:product_tm_orth} and Appendix~\ref{app:product_tm_orth_full_proof}, the family
\[
    \{B_k(\cdot;a)\}_{k\in[K]^d}
\]
is orthonormal in $H^2(\mathbb{T}^d)$. The separable product structure allows projection onto and reconstruction from the adaptive TM system to be implemented through one-dimensional operations along each coordinate axis.

\subsection{Analytic Branches and Adaptive Spectral Coefficients}
\label{app:branches}

\paragraph{Analytic branches.}
The Hardy space $H^2(\mathbb{T}^d)$ contains functions with nonnegative-frequency Fourier support, whereas real-valued PDE fields generally have Fourier coefficients on the full lattice $\mathbb{Z}^d$. Following the main text, SAFDNO retains the $+$ branch in the first coordinate and recovers the remaining frequency support through analytic branches indexed by
\begin{equation}
    \Sigma_1
    =
    \{+\},
    \qquad
    \Sigma_2
    =
    \{++, +-\},
    \qquad
    \Sigma_3
    =
    \{+++, ++-, +-+, +--\}.
    \label{eq:branch_set_app}
\end{equation}
Equivalently,
\begin{equation}
    \Sigma_d
    =
    \{+\}\times\{+,-\}^{d-1}.
\end{equation}
For real-valued signals, the branches not included in $\Sigma_d$ are determined by conjugate symmetry. Therefore, averaging over $\Sigma_d$ and taking the real part returns a real-valued signal in the physics domain.

\paragraph{Branch operators.}
For each $r=1,\dots,d$, let $\mathcal{H}_r$ denote the periodic Hilbert transform in the $r$-th variable, and let $\mathcal{R}_r$ denote the periodic coordinate reversal in the $r$-th coordinate. If $z_r=e^{\mathrm{i}t_r}$, then
\begin{equation}
    (\mathcal{R}_r g)(t_1,\dots,t_r,\dots,t_d)
    =
    g(t_1,\dots,-t_r,\dots,t_d).
\end{equation}
Define the coordinate-wise analytic branch operators by
\begin{equation}
    \mathcal{A}^{(r)}_+
    =
    I+\mathrm{i}\mathcal{H}_r,
    \qquad
    \mathcal{A}^{(r)}_-
    =
    \mathcal{R}_r\bigl(I-\mathrm{i}\mathcal{H}_r\bigr).
    \label{eq:branch_ops_app}
\end{equation}
The operator $\mathcal{A}^{(r)}_-$ forms the anti-analytic branch in the $r$-th coordinate and then applies the coordinate reversal $t_r\mapsto -t_r$, so that the resulting branch is represented in the nonnegative-frequency Hardy space. For $\sigma=(\sigma_1,\dots,\sigma_d)\in\Sigma_d$, the composite branch operator is
\begin{equation}
    \mathcal{A}_\sigma
    =
    \mathcal{A}^{(1)}_{\sigma_1}
    \mathcal{A}^{(2)}_{\sigma_2}
    \cdots
    \mathcal{A}^{(d)}_{\sigma_d},
    \qquad
    v^\sigma
    =
    \mathcal{A}_\sigma v.
\end{equation}
As established in Appendix~\ref{app:afd_kernel_bounded_full}, $v^\sigma\in H^2(\mathbb{T}^d;\mathbb{C}^{d_v})$ for every $\sigma\in\Sigma_d$, so each retained branch admits an expansion in the product-TM system.

\paragraph{Adaptive spectral coefficients.}
For $\sigma\in\Sigma_d$ and $k\in[K]^d$, the adaptive spectral coefficient is
\begin{equation}
    C^\sigma_k
    =
    \bigl\langle
        v^\sigma,
        B_k(\cdot;a)
    \bigr\rangle_{H^2(\mathbb{T}^d)} ,
    \label{eq:adaptive_coeff_app}
\end{equation}
where the inner product is taken channel-wise. Thus $C^\sigma_k\in\mathbb{C}^{d_v}$ is the complex expansion coefficient associated with the $k$-th adaptive TM mode on analytic branch $\sigma$. In contrast to FNO, SAFDNO does not rely on a fixed Fourier basis with hard frequency truncation. The adaptive spectral coefficients are computed with respect to the data-dependent TM system induced by the predicted pole collection $a=\Psi_\theta(v)$. The size of the adaptive spectral representation is controlled by the number of predicted poles along each spatial dimension.

\subsection{SAFD Kernel Integral Operator}
\label{app:safd_operator}

\paragraph{Learned filtering in the adaptive spectral domain.}
For the coefficients in the adaptive spectral domain, SAFDNO defines a neural network for filtering. Specifically, for each $k\in[K]^d$, we introduce a learnable linear transform
\begin{equation}
    R_k(a;\theta)
    :=
    \psi_\theta(k;a)
    \in
    \mathbb{C}^{d_v\times d_v},
    \label{eq:mode_filter_app}
\end{equation}
where $\psi_\theta$ is a neural network conditioned on the pole collection $a$. In this work, $\psi_\theta$ is implemented as a sinusoidal representation network (SIREN). The matrix $R_k(a;\theta)$ acts on the channel dimension of the adaptive spectral coefficient $C^\sigma_k$, while different adaptive TM modes are filtered independently in the coefficient domain.

\paragraph{Branch-wise kernel operator.}
For a fixed pole collection $a$, the branch-wise SAFD map can be written as a finite-rank kernel integral operator. Let $d\mu$ denote the normalized measure on $\mathbb{T}^d$. Define
\begin{equation}
    \kappa_{a,\theta}(z,\zeta)
    =
    \sum_{k\in[K]^d}
    B_k(z;a)\,
    R_k(a;\theta)\,
    \overline{B_k(\zeta;a)},
    \qquad
    z,\zeta\in\mathbb{T}^d .
    \label{eq:branch_kernel_app}
\end{equation}
Then the branch-wise reconstruction is
\begin{equation}
    u^\sigma(z)
    :=
    \sum_{k\in[K]^d}
    R_k(a;\theta)\,
    C^\sigma_k\,
    B_k(z;a)
    =
    \int_{\mathbb{T}^d}
    \kappa_{a,\theta}(z,\zeta)\,
    v^\sigma(\zeta)\,
    d\mu(\zeta).
    \label{eq:branch_recon_app}
\end{equation}
Thus, for each retained analytic branch, SAFDNO applies learned filtering to adaptive TM coefficients and reconstructs the result in the corresponding input-dependent TM system.

\paragraph{Full SAFD kernel operator.}
For branches with $\sigma_r=-$, the corresponding analytic branch is reconstructed in a coordinate-reversed Hardy representation. To map the reconstructed branch back to the original coordinate orientation, define the combined coordinate-reversal operator
\begin{equation}
    \mathcal{R}_\sigma
    :=
    \prod_{r:\,\sigma_r=-}
    \mathcal{R}_r,
\end{equation}
with the convention that an empty product is the identity. The fully expanded SAFD kernel integral operator is
\begin{equation}
    \bigl(K(a;\theta)v\bigr)(z)
    =
    \Re\!\left[
        \frac{1}{|\Sigma_d|}
        \sum_{\sigma\in\Sigma_d}
        \bigl(\mathcal{R}_\sigma u^\sigma\bigr)(z)
    \right].
    \label{eq:safd_op_app}
\end{equation}
Equation~\eqref{eq:safd_op_app} expands the compact branch notation in Eq.~\eqref{eq:afd_operator} and makes explicit the coordinate-reversal correction for branches with $\sigma_r=-$. Since each $\mathcal{R}_r$ is a unitary involution on $L^2(\mathbb{T}^d)$, the boundedness result of Proposition~\ref{prop:afd_kernel_bounded} remains unchanged:
\begin{equation}
    \|K(a;\theta)v\|_{L^2}
    \le
    2^d M(a,\theta)\|v\|_{L^2},
    \qquad
    M(a,\theta)
    :=
    \max_{k\in[K]^d}
    \|R_k(a;\theta)\|_{\mathrm{op}} .
\end{equation}

\paragraph{Specialization to $d=1,2,3$.}
For $d=1$, the only retained branch is $\sigma=+$ and no coordinate reversal is required:
\begin{equation}
    \bigl(K(a;\theta)v\bigr)(z)
    =
    \Re\!\left[
        \sum_{k=1}^K
        R_k(a;\theta)
        C^+_k
        B_k(z;a)
    \right].
\end{equation}
For $d=2$,
\begin{equation}
    \bigl(K(a;\theta)v\bigr)(z)
    =
    \Re\!\left[
        \frac{1}{2}
        \left(
            u^{++}(z)
            +
            \bigl(\mathcal{R}_2 u^{+-}\bigr)(z)
        \right)
    \right].
\end{equation}
For $d=3$,
\begin{equation}
    \bigl(K(a;\theta)v\bigr)(z)
    =
    \Re\!\left[
        \frac{1}{4}
        \left(
            u^{+++}(z)
            +
            \bigl(\mathcal{R}_3 u^{++-}\bigr)(z)
            +
            \bigl(\mathcal{R}_2 u^{+-+}\bigr)(z)
            +
            \bigl(\mathcal{R}_2\mathcal{R}_3 u^{+--}\bigr)(z)
        \right)
    \right].
\end{equation}
In each case, the operator decomposes the input into retained analytic branches, projects each branch onto the product-TM system, applies learned filtering through $R_k(a;\theta)$, reconstructs each branch, maps coordinate-reversed branches back to the original coordinate orientation, averages over the retained branches, and takes the real part.

\subsection{SAFDNO Operator Layer}
\label{app:layer}

\paragraph{Basic layer.}
Let $v^{(l)}$ denote the latent representation entering the $l$-th SAFDNO layer, and let
\begin{equation}
    a^{(l)}
    =
    \Psi^{(l)}_\theta\!\left(v^{(l)}\right)
\end{equation}
be the pole collection predicted for that layer. Writing
\begin{equation}
    s^{(l)}(z)
    :=
    \bigl(K(a^{(l)};\theta_r)v^{(l)}\bigr)(z),
\end{equation}
the basic SAFDNO layer is
\begin{equation}
    v^{(l+1)}(z)
    =
    v^{(l)}(z)
    +
    W_2^{(l)}
    \phi\!\left(
        W_1^{(l)}s^{(l)}(z)
        +
        b_1^{(l)}
    \right)
    +
    b_2^{(l)},
    \label{eq:layer_basic_app}
\end{equation}
where $\phi$ is a pointwise activation function.

\paragraph{Augmented layer for structured-mesh PDEs.}
For structured-mesh PDEs, SAFDNO may augment the SAFD spectral output with local differential features computed by central finite differences. Let $e_r$ denote the unit index vector along the $r$-th spatial dimension. At grid index $n$, define the discrete first-order central difference operator and the discrete Laplacian operator by
\begin{equation}
    (\delta_r f)[n]
    =
    \frac{1}{2}
    \bigl(
        f[n+e_r]-f[n-e_r]
    \bigr),
    \qquad
    (\Delta_h f)[n]
    =
    \sum_{r=1}^d
    \bigl(
        f[n+e_r]+f[n-e_r]-2f[n]
    \bigr),
    \label{eq:fd_operators_app}
\end{equation}
with periodic boundary handling. The augmented SAFDNO layer is
\begin{equation}
    v^{(l+1)}(z)
    =
    v^{(l)}(z)
    +
    W_2^{(l)}
    \phi\!\left(
        W_1^{(l)}
        \bigl[
            s^{(l)}(z),
            \{\delta_r s^{(l)}(z)\}_{r=1}^d,
            \Delta_h s^{(l)}(z)
        \bigr]
        +
        b_1^{(l)}
    \right)
    +
    b_2^{(l)},
    \label{eq:layer_aug_app}
\end{equation}
where $[\cdot]$ denotes concatenation along the channel dimension. The local differential feature tuple specializes to
\begin{align}
    d=1:\quad
    &
    \bigl(
        s^{(l)},
        \delta_1 s^{(l)},
        \Delta_h s^{(l)}
    \bigr),
    \\
    d=2:\quad
    &
    \bigl(
        s^{(l)},
        \delta_1 s^{(l)},
        \delta_2 s^{(l)},
        \Delta_h s^{(l)}
    \bigr),
    \\
    d=3:\quad
    &
    \bigl(
        s^{(l)},
        \delta_1 s^{(l)},
        \delta_2 s^{(l)},
        \delta_3 s^{(l)},
        \Delta_h s^{(l)}
    \bigr).
\end{align}
Both the basic layer~\eqref{eq:layer_basic_app} and the augmented layer~\eqref{eq:layer_aug_app} preserve the spatial resolution of the latent representation on the working grid. The SAFD spectral component is parameterized by the neural pole predictor $\Psi_\theta$ and the filtering network $\psi_\theta(k;a)$, rather than by a resolution-indexed tensor of Fourier weights. Therefore, after discretizing the TM projection and reconstruction operators on a target grid, the same learned spectral parameterization can be evaluated across different regular-grid resolutions. The finite-difference terms in Eq.~\eqref{eq:layer_aug_app} should be interpreted as local differential features for structured-mesh PDEs.

\section{Proofs and Theoretical Analysis}\label{app:proofs}

\subsection{Justification of Remark~\ref{rem:tm_fourier_1d}}\label{app:remark1_full}
\begin{proof}
If $a_j=0$ for all $j$, then Eq.~\eqref{eq:1d_tm_system} gives
\[
    B_k(z;0)
    =\frac{\sqrt{1-|0|^2}}{1-\overline{0}z}
      \prod_{j=1}^{k-1}\frac{z-0}{1-\overline{0}z}
    =z^{k-1},
    \qquad k\ge 1.
\]
Hence
\[
    B_1(z;0)=1,
    \qquad
    B_k(z;0)=z^{k-1},\qquad k\ge 2,
\]
so the TM system reduces to the nonnegative-frequency Fourier basis of $H^2(\mathbb{T})$.

Now let $a=\rho e^{i\phi}$ with $0\le \rho<1$. On the unit circle, $z=e^{it}$, one has
\[
    e_a(e^{it})
    =\frac{\sqrt{1-\rho^2}}{1-\rho e^{i(t-\phi)}},
\]
and therefore
\[
    |e_a(e^{it})|^2
    =\frac{1-\rho^2}{|1-\rho e^{i(t-\phi)}|^2}
    =\frac{1-\rho^2}{1+\rho^2-2\rho\cos(t-\phi)}.
\]
The denominator is minimized when $\cos(t-\phi)=1$, namely at $t=\phi \pmod{2\pi}$, so $|e_a(e^{it})|$ attains its maximum there. Moreover,
\[
    |e_a(e^{i\phi})|^2
    =\frac{1-\rho^2}{(1-\rho)^2}
    =\frac{1+\rho}{1-\rho}
    \longrightarrow +\infty
    \qquad (\rho\to 1^-).
\]
On the other hand, for any fixed $\delta>0$, if $|t-\phi|\ge \delta$, then
\[
    |e_a(e^{it})|^2
    \le
    \frac{1-\rho^2}{1+\rho^2-2\rho\cos\delta}
    \longrightarrow 0
    \qquad (\rho\to 1^-).
\]
Thus, as the pole approaches the unit circle, the atom becomes increasingly concentrated near the angle $\phi$.
\end{proof}

\subsection{Proof of Proposition~\ref{prop:product_tm_orth}} \label{app:product_tm_orth_full_proof}
\begin{proof}
For each $r\in\{1,\dots,d\}$, the one-dimensional TM system
\[
    \{B_{k_r}(\cdot;a^{(r)})\}_{k_r=1}^K
\]
is orthonormal in $H^2(\mathbb{T})$; see, for instance, \citep{AFD_Qian,2DAFD_Qian}. Let
\[
    k=(k_1,\dots,k_d),
    \qquad
    \ell=(\ell_1,\dots,\ell_d)\in[K]^d.
\]
Using Eq.~\eqref{eq:product_tm} and Fubini's theorem, we obtain
\begin{align*}
    \langle B_k(\cdot;a),B_\ell(\cdot;a)\rangle_{H^2(\mathbb{T}^d)}
    &=\frac{1}{(2\pi)^d}\int_{[0,2\pi]^d}
      \prod_{i=1}^d B_{k_r}(e^{it_r};a^{(r)})
      \overline{B_{\ell_r}(e^{it_r};a^{(r)})}\,
      dt_1\cdots dt_d\\
    &=\prod_{i=1}^d
      \left(
      \frac{1}{2\pi}\int_0^{2\pi}
      B_{k_r}(e^{it_r};a^{(r)})
      \overline{B_{\ell_r}(e^{it_r};a^{(r)})}\,dt_r
      \right)\\
    &=\prod_{i=1}^d
      \langle B_{k_r}(\cdot;a^{(r)}),B_{\ell_r}(\cdot;a^{(r)})\rangle_{H^2(\mathbb{T})}\\
    &=\prod_{i=1}^d \delta_{k_r\ell_r}.
\end{align*}
This proves that $\{B_k(\cdot;a)\}_{k\in[K]^d}$ is an orthonormal family in $H^2(\mathbb{T}^d)$. The stated Bessel inequality follows immediately,
\[
    \sum_{k\in[K]^d}
    \bigl|\langle v,B_k(\cdot;a)\rangle_{H^2(\mathbb{T}^d)}\bigr|^2
    \le
    \|v\|_{H^2(\mathbb{T}^d)}^2,
    \qquad v\in H^2(\mathbb{T}^d).
\]
\end{proof}

\subsection{Proof of Proposition~\ref{prop:afd_kernel_bounded}}\label{app:afd_kernel_bounded_full}
\begin{proof}
Recall that for each $r\in\{1,\dots,d\}$,
\[
    \mathcal{A}_+^{(r)}=I+i\mathcal{H}_r,
    \qquad
    \mathcal{A}_-^{(r)}=\mathcal{R}_r(I-i\mathcal{H}_r),
\]
where $\mathcal{H}_r$ is the periodic Hilbert transform in the $r$-th variable and
\[
    (\mathcal{R}_r g)(t_1,\dots,t_r,\dots,t_d)
    =g(t_1,\dots,-t_r,\dots,t_d).
\]
For $f\in L^2(\mathbb{T}^d)$, write
\[
    f(t)=\sum_{n\in\mathbb{Z}^d}\widehat f(n)e^{in\cdot t},
    \qquad
    n\cdot t:=\sum_{r=1}^d n_r t_r,
\]
so that
\[
    \|f\|_{L^2(\mathbb{T}^d)}^2
    =\sum_{n\in\mathbb{Z}^d}|\widehat f(n)|^2.
\]
Fix $\sigma=(\sigma_1,\dots,\sigma_d)\in\Sigma_d$, and define
\[
    \varepsilon_r=
    \begin{cases}
        +1,& \sigma_r=+,\\
        -1,& \sigma_r=-,
    \end{cases}
    \qquad
    \varepsilon_\sigma\odot m:=(\varepsilon_1m_1,\dots,\varepsilon_dm_d).
\]
The multiplier formula for the periodic Hilbert transform gives
\[
    \widehat{\mathcal{H}_r f}(n)
    =-i\,\operatorname{sgn}(n_r)\widehat f(n),
    \qquad n\in\mathbb{Z}^d.
\]
Hence
\[
    \widehat{\mathcal{A}_+^{(r)}f}(n)
    =\widehat{(I+i\mathcal{H}_r)f}(n)
    =(1+\operatorname{sgn}(n_r))\widehat f(n).
\]
The reflection operator satisfies
\[
    \widehat{\mathcal{R}_r g}(n)
    =\widehat g(n_1,\dots,-n_r,\dots,n_d),
\]
and therefore
\begin{align*}
    \widehat{\mathcal{A}_-^{(r)}f}(n)
    &=\widehat{\mathcal{R}_r(I-i\mathcal{H}_r)f}(n)\\
    &=\widehat{(I-i\mathcal{H}_r)f}(n_1,\dots,-n_r,\dots,n_d)\\
    &=(1-\operatorname{sgn}(-n_r))\widehat f(n_1,\dots,-n_r,\dots,n_d)\\
    &=(1+\operatorname{sgn}(n_r))\widehat f(n_1,\dots,-n_r,\dots,n_d).
\end{align*}
For each coordinate direction $r$, the operator $\mathcal A_{\sigma_r}^{(r)}$ acts only on the variable $t_r$. Thus, in the $d$-dimensional setting, the action of $\mathcal A_\sigma$ separates across coordinates, that is, each factor contributes its own one-dimensional Fourier multiplier, and the total multiplier is the product of these coordinate contributions. Therefore,
\[
    \widehat{\mathcal A_\sigma f}(m)
    =
    \left(\prod_{r=1}^d (1+\operatorname{sgn}(m_r))\right)
    \widehat f(\varepsilon_\sigma\odot m).
\]
If some coordinates are $m_r<0$, then the factor $1+\operatorname{sgn}(m_r)$ vanishes. Thus $\widehat{\mathcal{A}_\sigma f}(m)=0$ whenever $m\notin\mathbb{N}_0^d$, and therefore $\mathcal{A}_\sigma f\in H^2(\mathbb{T}^d)$. By Plancherel,
\begin{align*}
    \|\mathcal{A}_\sigma f\|_{L^2(\mathbb{T}^d)}^2
    &=
    \sum_{m\in\mathbb{N}_0^d}
    \left|
    \left(\prod_{r=1}^d (1+\operatorname{sgn}(m_r))\right)
    \widehat f(\varepsilon_\sigma\odot m)
    \right|^2\\
    &\le
    4^d
    \sum_{m\in\mathbb{N}_0^d}
    |\widehat f(\varepsilon_\sigma\odot m)|^2\\
    &\le
    4^d
    \sum_{n\in\mathbb{Z}^d}|\widehat f(n)|^2
    =4^d\|f\|_{L^2(\mathbb{T}^d)}^2.
\end{align*}
Consequently,
\[
    \|\mathcal{A}_\sigma f\|_{L^2(\mathbb{T}^d)}
    \le 2^d\|f\|_{L^2(\mathbb{T}^d)}.
\]

Now let $v\in L^2(\mathbb{T}^d;\mathbb{R}^{d_v})$ and set $v^\sigma=\mathcal{A}_\sigma v$, applying the previous estimate componentwise. Then
\[
    v^\sigma\in H^2(\mathbb{T}^d;\mathbb{C}^{d_v}),
    \qquad
    \|v^\sigma\|_{L^2}\le 2^d\|v\|_{L^2}.
\]
Hence, the coefficients
\[
    C_k^\sigma
    =\langle v^\sigma,B_k(\cdot;a)\rangle_{H^2},
    \qquad k\in[K]^d,
\]
are well-defined. Since $\{B_k(\cdot;a)\}_{k\in[K]^d}$ is orthonormal in $H^2(\mathbb{T}^d)$ by Proposition~\ref{prop:product_tm_orth}, the applied Bessel inequality component-wise gives
\begin{align*}
    \sum_{k\in[K]^d}\|C_k^\sigma\|_2^2
    &=
    \sum_{j=1}^{d_v}\sum_{k\in[K]^d}
    \left|
    \left\langle (v^\sigma)_j,B_k(\cdot;a)\right\rangle_{H^2(\mathbb{T}^d)}
    \right|^2\\
    &\le
    \sum_{j=1}^{d_v}\|(v^\sigma)_j\|_{L^2(\mathbb{T}^d)}^2
    =\|v^\sigma\|_{L^2(\mathbb{T}^d)}^2\\
    &\le 4^d\|v\|_{L^2(\mathbb{T}^d)}^2.
\end{align*}

Define
\[
    u^\sigma(z)
    :=
    \sum_{k\in[K]^d}R_k(a;\theta)\,C_k^\sigma\,B_k(z;a).
\]
Because the sum is finite, $u^\sigma\in H^2(\mathbb{T}^d;\mathbb{C}^{d_v})$. Using orthonormality,
\begin{align*}
    \|u^\sigma\|_{L^2}^2
    &=
    \sum_{j=1}^{d_v}
    \left\|
    \sum_{k\in[K]^d}(R_k(a;\theta)C_k^\sigma)_j\,B_k(\cdot;a)
    \right\|_{L^2}^2\\
    &=
    \sum_{j=1}^{d_v}\sum_{k\in[K]^d}
    |(R_k(a;\theta)C_k^\sigma)_j|^2\\
    &=
    \sum_{k\in[K]^d}\|R_k(a;\theta)C_k^\sigma\|_2^2\\
    &\le
    M(a,\theta)^2
    \sum_{k\in[K]^d}\|C_k^\sigma\|_2^2\\
    &\le
    4^d M(a,\theta)^2\|v\|_{L^2}^2.
\end{align*}
Therefore,
\[
    \|u^\sigma\|_{L^2}
    \le 2^d M(a,\theta)\|v\|_{L^2}.
\]
Finally,
\[
    \mathcal{K}(a;\theta)v
    =
    \Re\left(
    \frac{1}{|\Sigma_d|}
    \sum_{\sigma\in\Sigma_d}u^\sigma
    \right).
\]
Since $\|\Re w\|_2\le \|w\|_2$ for every $w\in\mathbb{C}^{d_v}$, the triangle inequality yields
\begin{align*}
    \|\mathcal{K}(a;\theta)v\|_{L^2}
    &\le
    \left\|
    \frac{1}{|\Sigma_d|}
    \sum_{\sigma\in\Sigma_d}u^\sigma
    \right\|_{L^2}\\
    &\le
    \frac{1}{|\Sigma_d|}
    \sum_{\sigma\in\Sigma_d}\|u^\sigma\|_{L^2}\\
    &\le
    2^d M(a,\theta)\|v\|_{L^2}.
\end{align*}
This proves boundedness. The preceding argument also shows that every term in Eq.~\eqref{eq:afd_operator} is well-defined for $v\in L^2(\mathbb{T}^d;\mathbb{R}^{d_v})$, so $\mathcal{K}(a;\theta)$ is well-defined. Moreover, all operations before the final real part are complex-linear, and $\Re(\cdot)$ is linear over $\mathbb{R}$. Hence $\mathcal{K}(a;\theta)$ is $\mathbb{R}$-linear as a map from $L^2(\mathbb{T}^d;\mathbb{R}^{d_v})$ to itself.
\end{proof}

\subsection{AFD approximation and convergence rate}

The following deterministic AFD theorem summarizes the one-dimensional approximation identity, interpolation property, convergence, and convergence-rate guarantee underlying the TM representation used in SAFDNO. The finite-section identity holds for arbitrary finite pole sequences, while the convergence and rate conclusions are stated for the classical maximal-selection AFD. The argument follows the generalized backward-shift construction of classical AFD \citep{AFD_Qian,2DAFD_Qian}; stochastic AFD provides the corresponding expectation-level pole-selection motivation \citep{qian2022sparse}.

\begin{theorem}[AFD approximation, interpolation, and convergence rate]
Let \(f\in H^2(\mathbb{D})\) and set \(f_1=f\). For \(a\in\mathbb{D}\), define
\[
    e_a(z)=\frac{\sqrt{1-|a|^2}}{1-\overline{a}z},
    \qquad
    b_a(z)=\frac{z-a}{1-\overline{a}z}.
\]
Given any finite pole sequence \(a_1,\ldots,a_n\in\mathbb{D}\), define the reduced remainders recursively by
\[
    f_{k+1}(z)
    =
    \frac{
        f_k(z)
        -
        \langle f_k,e_{a_k}\rangle_{H^2(\mathbb{D})}e_{a_k}(z)
    }
    {b_{a_k}(z)},
    \qquad k=1,\ldots,n,
\]
and define the associated Takenaka--Malmquist system by
\[
    B_k(z;\mathbf a)
    =
    e_{a_k}(z)\prod_{j=1}^{k-1}b_{a_j}(z),
    \qquad k=1,\ldots,n,
\]
where empty products are interpreted as \(1\). Then \(f_{k+1}\in H^2(\mathbb{D})\) for every \(k=1,\ldots,n\), and
\[
    f(z)
    =
    \sum_{k=1}^{n}
    \langle f_k,e_{a_k}\rangle_{H^2(\mathbb{D})}
    B_k(z;\mathbf a)
    +
    f_{n+1}(z)\prod_{k=1}^{n}b_{a_k}(z).
\]
Consequently, if a pole \(\alpha\in\mathbb{D}\) occurs in \(\{a_1,\ldots,a_n\}\) with multiplicity \(m_\alpha\), then the partial sum
\[
    S_n f(z)
    :=
    \sum_{k=1}^{n}
    \langle f_k,e_{a_k}\rangle_{H^2(\mathbb{D})}
    B_k(z;\mathbf a)
\]
satisfies the Hermite-type interpolation conditions
\[
    \partial_z^q\bigl(f-S_n f\bigr)(\alpha)=0,
    \qquad q=0,\ldots,m_\alpha-1.
\]
In particular, for every selected pole \(a_\ell\),
\[
    f(a_\ell)
    =
    \sum_{k=1}^{n}
    \langle f_k,e_{a_k}\rangle_{H^2(\mathbb{D})}
    B_k(a_\ell;\mathbf a).
\]

If the construction is continued indefinitely and, whenever the current reduced remainder is nonzero, the poles are selected by the classical maximal selection principle
\[
    a_k
    \in
    \operatorname*{arg\,max}_{a\in\mathbb{D}}
    \bigl|
        \langle f_k,e_a\rangle_{H^2(\mathbb{D})}
    \bigr|^2,
    \qquad k=1,2,\ldots,
\]
then
\[
    f
    =
    \sum_{k=1}^{\infty}
    \langle f_k,e_{a_k}\rangle_{H^2(\mathbb{D})}
    B_k(\cdot;\mathbf a)
    \quad\text{in }H^2(\mathbb{D}).
\]
Moreover, for \(0<M<\infty\), define
\[
    H^2(\mathbb{D},M)
    :=
    \left\{
    f\in H^2(\mathbb{D}):
    f=\sum_{j=1}^{\infty}c_j e_{\beta_j},\
    \beta_j\in\mathbb{D},\
    \sum_{j=1}^{\infty}|c_j|\le M
    \right\}.
\]
If \(f\in H^2(\mathbb{D},M)\), then the maximal-selection AFD approximation satisfies
\[
    \left\|
    f-
    \sum_{k=1}^{n}
    \langle f_k,e_{a_k}\rangle_{H^2(\mathbb{D})}
    B_k(\cdot;\mathbf a)
    \right\|_{H^2(\mathbb{D})}
    \le
    \frac{M}{\sqrt n},
    \qquad n\ge 1.
\]
\end{theorem}

\begin{proof}
Let
\[
    \phi_0(z)\equiv 1,
    \qquad
    \phi_k(z)=\prod_{j=1}^{k}b_{a_j}(z),
    \qquad k\ge 1.
\]
The reduced-remainder recursion is well-defined in \(H^2(\mathbb{D})\). Indeed, by the reproducing property of the normalized Szeg\H{o} kernel,
\[
    \langle f_k,e_{a_k}\rangle_{H^2(\mathbb{D})}
    =
    \sqrt{1-|a_k|^2}\,f_k(a_k),
    \qquad
    e_{a_k}(a_k)
    =
    \frac{1}{\sqrt{1-|a_k|^2}}.
\]
Therefore
\[
    f_k(a_k)
    -
    \langle f_k,e_{a_k}\rangle_{H^2(\mathbb{D})}e_{a_k}(a_k)
    =
    0.
\]
Thus
\[
    h_k(z)
    :=
    f_k(z)
    -
    \langle f_k,e_{a_k}\rangle_{H^2(\mathbb{D})}e_{a_k}(z)
\]
has a zero at \(a_k\). Since \(b_{a_k}\) is a Blaschke factor, the quotient \(h_k/b_{a_k}\) is analytic in \(\mathbb{D}\). Moreover, \(b_{a_k}\) is inner, so \(|b_{a_k}(e^{it})|=1\) for almost every \(t\), and the quotient has finite \(H^2\)-norm. Hence \(f_{k+1}=h_k/b_{a_k}\in H^2(\mathbb{D})\).

The recursion is equivalent to the one-step decomposition
\[
    f_k(z)
    =
    \langle f_k,e_{a_k}\rangle_{H^2(\mathbb{D})}e_{a_k}(z)
    +
    b_{a_k}(z)f_{k+1}(z).
\]
Multiplying this identity by \(\phi_{k-1}\) gives
\[
    \phi_{k-1}(z)f_k(z)
    =
    \langle f_k,e_{a_k}\rangle_{H^2(\mathbb{D})}
    e_{a_k}(z)\phi_{k-1}(z)
    +
    \phi_k(z)f_{k+1}(z).
\]
Since \(B_k(z;\mathbf a)=e_{a_k}(z)\phi_{k-1}(z)\), summing from \(k=1\) to \(n\) telescopes and yields
\[
    f(z)
    =
    \sum_{k=1}^{n}
    \langle f_k,e_{a_k}\rangle_{H^2(\mathbb{D})}
    B_k(z;\mathbf a)
    +
    \phi_n(z)f_{n+1}(z).
\]
This is the desired finite-section identity.

The interpolation property follows immediately. Since
\[
    f-S_n f=\phi_n f_{n+1},
\]
and since \(\phi_n\) has a zero of order \(m_\alpha\) at every pole \(\alpha\) appearing \(m_\alpha\) times in the sequence \(a_1,\ldots,a_n\), the product \(\phi_n f_{n+1}\) has a zero of order at least \(m_\alpha\) at \(\alpha\). Therefore
\[
    \partial_z^q(f-S_n f)(\alpha)=0,
    \qquad q=0,\ldots,m_\alpha-1.
\]
The pointwise interpolation formula is the case \(q=0\).

It remains to record the orthogonality and energy identities needed for the convergence part. The one-dimensional TM functions \(B_k=e_{a_k}\phi_{k-1}\) form an orthonormal system. Normalization follows from \(\|e_{a_k}\|_{H^2(\mathbb{D})}=1\) and the fact that multiplication by a finite Blaschke product is an isometry on \(H^2(\mathbb{D})\). If \(j<k\), then
\[
    B_k
    =
    \phi_{j-1}b_{a_j}
    \left(\prod_{\ell=j+1}^{k-1}b_{a_\ell}\right)
    e_{a_k},
    \qquad
    B_j=\phi_{j-1}e_{a_j}.
\]
Using again that multiplication by \(\phi_{j-1}\) is an isometry,
\[
\begin{aligned}
    \langle B_k,B_j\rangle_{H^2(\mathbb{D})}
    &=
    \left\langle
    b_{a_j}
    \left(\prod_{\ell=j+1}^{k-1}b_{a_\ell}\right)
    e_{a_k},
    e_{a_j}
    \right\rangle_{H^2(\mathbb{D})} \\
    &=
    \sqrt{1-|a_j|^2}\,
    \left[
    b_{a_j}
    \left(\prod_{\ell=j+1}^{k-1}b_{a_\ell}\right)
    e_{a_k}
    \right](a_j)
    =
    0.
\end{aligned}
\]
Thus \(\{B_k\}_{k\ge1}\) is orthonormal.

Define the standard orthogonal remainder by
\[
    g_k
    :=
    f-
    \sum_{j=1}^{k-1}
    \langle f_j,e_{a_j}\rangle_{H^2(\mathbb{D})}B_j.
\]
The finite-section identity with \(k-1\) terms gives
\[
    g_k=\phi_{k-1}f_k.
\]
Using the orthogonality of \(B_k\) to \(B_1,\ldots,B_{k-1}\), one obtains
\[
\begin{aligned}
    \langle f,B_k\rangle_{H^2(\mathbb{D})}
    &=
    \langle g_k,B_k\rangle_{H^2(\mathbb{D})} \\
    &=
    \langle \phi_{k-1}f_k,\phi_{k-1}e_{a_k}\rangle_{H^2(\mathbb{D})} \\
    &=
    \langle f_k,e_{a_k}\rangle_{H^2(\mathbb{D})}.
\end{aligned}
\]
The one-step decomposition is orthogonal because
\[
    \langle b_{a_k}f_{k+1},e_{a_k}\rangle_{H^2(\mathbb{D})}
    =
    \sqrt{1-|a_k|^2}\,
    b_{a_k}(a_k)f_{k+1}(a_k)
    =
    0.
\]
Since multiplication by \(b_{a_k}\) is an isometry,
\[
    \|f_k\|_{H^2(\mathbb{D})}^2
    =
    \bigl|
        \langle f_k,e_{a_k}\rangle_{H^2(\mathbb{D})}
    \bigr|^2
    +
    \|f_{k+1}\|_{H^2(\mathbb{D})}^2.
\]
Iterating this identity gives
\[
    \|f\|_{H^2(\mathbb{D})}^2
    =
    \sum_{k=1}^{n}
    \bigl|
        \langle f_k,e_{a_k}\rangle_{H^2(\mathbb{D})}
    \bigr|^2
    +
    \|f_{n+1}\|_{H^2(\mathbb{D})}^2.
\]
Also, by the finite-section identity,
\[
    \|f-S_n f\|_{H^2(\mathbb{D})}
    =
    \|\phi_n f_{n+1}\|_{H^2(\mathbb{D})}
    =
    \|f_{n+1}\|_{H^2(\mathbb{D})}.
\]

We now prove convergence under the maximal selection principle. If \(f_N=0\) for some finite \(N\), then the expansion terminates and the convergence statement is immediate. Hence we may assume \(f_n\neq0\) for all \(n\). Let
\[
    \mathcal M
    =
    \overline{\operatorname{span}}\{B_k:k\ge1\}
    \subset H^2(\mathbb{D}).
\]
Since the \(B_k\)'s form an orthonormal system and the coefficient of \(B_k\) in \(S_n f\) is \(\langle f,B_k\rangle_{H^2(\mathbb{D})}\), \(S_n f\) is the orthogonal projection of \(f\) onto \(\operatorname{span}\{B_1,\ldots,B_n\}\). Hence
\[
    g_{n+1}:=f-S_n f
\]
converges in \(H^2(\mathbb{D})\) to
\[
    g:=f-P_{\mathcal M}f\in\mathcal M^\perp.
\]
Assume for contradiction that \(g\neq0\). The Szeg\H{o} dictionary \(\{e_b:b\in\mathbb{D}\}\) is total in \(H^2(\mathbb{D})\): if \(\langle h,e_b\rangle_{H^2(\mathbb{D})}=0\) for every \(b\in\mathbb{D}\), then \(h(b)=0\) for every \(b\in\mathbb{D}\), and therefore \(h\equiv0\). Thus there exists \(b\in\mathbb{D}\) such that
\[
    \delta
    :=
    \bigl|
        \langle g,e_b\rangle_{H^2(\mathbb{D})}
    \bigr|
    >
    0.
\]
For all sufficiently large \(n\),
\[
    \bigl|
        \langle g_n,e_b\rangle_{H^2(\mathbb{D})}
    \bigr|
    \ge
    \frac{\delta}{2}.
\]
Since \(g_n=\phi_{n-1}f_n\) and \(|\phi_{n-1}(b)|\le1\),
\[
\begin{aligned}
    \bigl|
        \langle f_n,e_b\rangle_{H^2(\mathbb{D})}
    \bigr|
    &=
    \sqrt{1-|b|^2}\,|f_n(b)| \\
    &\ge
    \sqrt{1-|b|^2}\,|\phi_{n-1}(b)|\,|f_n(b)| \\
    &=
    \bigl|
        \langle g_n,e_b\rangle_{H^2(\mathbb{D})}
    \bigr|
    \ge
    \frac{\delta}{2}.
\end{aligned}
\]
The maximal selection principle gives
\[
    \bigl|
        \langle f_n,e_{a_n}\rangle_{H^2(\mathbb{D})}
    \bigr|
    \ge
    \bigl|
        \langle f_n,e_b\rangle_{H^2(\mathbb{D})}
    \bigr|
    \ge
    \frac{\delta}{2}
\]
for all sufficiently large \(n\). On the other hand,
\[
    \langle f_n,e_{a_n}\rangle_{H^2(\mathbb{D})}
    =
    \langle f,B_n\rangle_{H^2(\mathbb{D})}.
\]
By Bessel's inequality for the orthonormal system \(\{B_n\}\),
\[
    \sum_{n=1}^{\infty}
    \bigl|
        \langle f,B_n\rangle_{H^2(\mathbb{D})}
    \bigr|^2
    \le
    \|f\|_{H^2(\mathbb{D})}^2.
\]
Therefore \(\langle f,B_n\rangle_{H^2(\mathbb{D})}\to0\), contradicting the positive lower bound above. Hence \(g=0\), and
\[
    \|f-S_n f\|_{H^2(\mathbb{D})}\to0.
\]

Finally, suppose \(f\in H^2(\mathbb{D},M)\). Choose an atomic representation
\[
    f
    =
    \sum_{j=1}^{\infty}c_j e_{\beta_j},
    \qquad
    \beta_j\in\mathbb{D},
    \qquad
    \sum_{j=1}^{\infty}|c_j|\le M.
\]
For \(a\in\mathbb{D}\), define the generalized backward shift
\[
    \mathcal S_a h
    :=
    \frac{
        h-\langle h,e_a\rangle_{H^2(\mathbb{D})}e_a
    }
    {b_a}.
\]
The one-step energy identity shows that \(\mathcal S_a\) is a contraction on \(H^2(\mathbb{D})\). A direct calculation gives, for any \(a,\beta\in\mathbb{D}\),
\[
    \mathcal S_a e_\beta
    =
    \lambda(a,\beta)e_\beta,
    \qquad
    \lambda(a,\beta)
    =
    \frac{\overline{\beta}-\overline{a}}{1-\overline{\beta}a}.
\]
Indeed,
\[
\begin{aligned}
    \mathcal S_a e_\beta(z)
    &=
    \frac{
    \dfrac{\sqrt{1-|\beta|^2}}{1-\overline{\beta}z}
    -
    \dfrac{\sqrt{1-|a|^2}\sqrt{1-|\beta|^2}}{1-\overline{\beta}a}
    \dfrac{\sqrt{1-|a|^2}}{1-\overline{a}z}
    }
    {\dfrac{z-a}{1-\overline{a}z}}  \\
    &=
    \frac{\overline{\beta}-\overline{a}}{1-\overline{\beta}a}
    \frac{\sqrt{1-|\beta|^2}}{1-\overline{\beta}z}
    =
    \lambda(a,\beta)e_\beta(z).
\end{aligned}
\]
Moreover,
\[
    |\lambda(a,\beta)|
    =
    \left|
        \frac{\beta-a}{1-\overline{a}\beta}
    \right|
    \le1.
\]
Since the atomic expansion is absolutely convergent in \(H^2(\mathbb{D})\), repeated application of the bounded operators \(\mathcal S_{a_1},\ldots,\mathcal S_{a_{n-1}}\) gives
\[
    f_n
    =
    \sum_{j=1}^{\infty}c_j^{(n)}e_{\beta_j},
    \qquad
    c_j^{(n)}
    =
    c_j
    \prod_{\ell=1}^{n-1}
    \lambda(a_\ell,\beta_j),
\]
and hence
\[
    \sum_{j=1}^{\infty}|c_j^{(n)}|
    \le
    \sum_{j=1}^{\infty}|c_j|
    \le M.
\]
In particular,
\[
    \|f_n\|_{H^2(\mathbb{D})}
    \le
    \sum_{j=1}^{\infty}|c_j^{(n)}|
    \le M.
\]
Taking inner products termwise,
\[
\begin{aligned}
    \|f_n\|_{H^2(\mathbb{D})}^2
    &=
    \bigl|
        \langle f_n,f_n\rangle_{H^2(\mathbb{D})}
    \bigr| \\
    &=
    \left|
    \sum_{j=1}^{\infty}
    c_j^{(n)}
    \langle e_{\beta_j},f_n\rangle_{H^2(\mathbb{D})}
    \right| \\
    &\le
    \sum_{j=1}^{\infty}
    |c_j^{(n)}|\,
    \bigl|
        \langle f_n,e_{\beta_j}\rangle_{H^2(\mathbb{D})}
    \bigr| \\
    &\le
    M
    \sup_{a\in\mathbb{D}}
    \bigl|
        \langle f_n,e_a\rangle_{H^2(\mathbb{D})}
    \bigr|.
\end{aligned}
\]
Under the maximal selection principle,
\[
    \sup_{a\in\mathbb{D}}
    \bigl|
        \langle f_n,e_a\rangle_{H^2(\mathbb{D})}
    \bigr|
    =
    \bigl|
        \langle f_n,e_{a_n}\rangle_{H^2(\mathbb{D})}
    \bigr|.
\]
Therefore
\[
    \bigl|
        \langle f_n,e_{a_n}\rangle_{H^2(\mathbb{D})}
    \bigr|
    \ge
    \frac{\|f_n\|_{H^2(\mathbb{D})}^2}{M}.
\]
Let
\[
    d_n:=\|f_n\|_{H^2(\mathbb{D})}^2.
\]
Combining the preceding lower bound with the one-step energy identity yields
\[
    d_{n+1}
    =
    d_n
    -
    \bigl|
        \langle f_n,e_{a_n}\rangle_{H^2(\mathbb{D})}
    \bigr|^2
    \le
    d_n\left(1-\frac{d_n}{M^2}\right).
\]
Since \(0\le d_n\le M^2\) for every \(n\), we claim that \(d_n\le M^2/n\). The claim is immediate for \(n=1\). If \(d_{n+1}=0\), it is also immediate. Otherwise, the recurrence implies \(0<d_n<M^2\), and hence
\[
    \frac{1}{d_{n+1}}
    \ge
    \frac{1}{d_n(1-d_n/M^2)}
    =
    \frac{1}{d_n}
    +
    \frac{1}{M^2-d_n}
    \ge
    \frac{1}{d_n}
    +
    \frac{1}{M^2}.
\]
Induction gives \(1/d_n\ge n/M^2\), and hence \(d_n\le M^2/n\). Therefore
\[
\begin{aligned}
    \left\|
    f-
    \sum_{k=1}^{n}
    \langle f_k,e_{a_k}\rangle_{H^2(\mathbb{D})}B_k
    \right\|_{H^2(\mathbb{D})}
    &=
    \|f_{n+1}\|_{H^2(\mathbb{D})} \\
    &=
    \sqrt{d_{n+1}}
    \le
    \frac{M}{\sqrt{n+1}}
    \le
    \frac{M}{\sqrt n}.
\end{aligned}
\]
This completes the proof.
\end{proof}

\paragraph{Remark on the convergence rate.}
Regarding the convergence rate, we note that functional approximation rate results are usually given in terms of the degree of smoothness of the function to be approximated: the higher the smoothness, the higher the convergence rate. For functions with analytic regularity, AFD, like other well-established approximation methods, has an exponential decay rate \citep{AFD_Qian}. The convergence rate \(O(n^{-1/2})\) above is actually meant to be rather good, for it gives a borderline decay rate for the worst cases, including non-smooth boundary behavior and even discontinuities. As useful references, exact Shannon interpolation and the Karhunen--Lo\`eve expansion obey the same convergence rate. On the other hand, in practical applications, the residual energy of AFD is observed to decay rather fast.

\clearpage
\section{Detailed Experiment Setups}\label{app:detaled_exp_setups}
This appendix provides additional experiment setups following Section~\ref{sec:implementation}. For all SAFDNO experiments, we use the same SIREN hyperparameter setting. Specifically, the coordinate input is first projected into a $16$-channel learnable Gaussian Fourier feature embedding with scale $\sigma=512$, and the resulting embedding is fed into the SIREN with one hidden layer, $32$ hidden channels, and $w=30$. The pole predictor is implemented as an attention-based neural network with $4$ attention heads, $2$ attention layers, and $8$ sinusoidal frequencies for positional encoding. Table~\ref{tab:datasets} lists the PDE datasets used in our main experiments, and Table~\ref{tab:hyperparams} presents the hyperparameter settings for our SAFDNO and FNO.

\begin{table}[htbp]
  \centering
  \caption{Description of benchmark datasets used in the main experiments.}
  \label{tab:datasets}
  \begin{tabular}{l|c c c c}
  \toprule
  \multicolumn{1}{c|}{\textsc{Dataset}} & \textsc{Resolution} & \textsc{No.\ train} & \textsc{No.\ test} & \textsc{Source}\\
  \midrule
  Darcy      & $85\!\times\!85$  & $1000$ & $200$
   & \citep{fno,kossaifi2025librarylearningneuraloperators}\\
  NS         & $64\!\times\!64,\;T=20$ & $1000$ & $200$
  & \citep{fno}\\
  Pipe       & $129\!\times\!129$  & $1000$ & $200$
  & \citep{geofno}\\
  Airfoil    & $221\!\times\!51$  & $1000$ & $200$
   & \citep{geofno}\\
  Plasticity & $101\!\times\!31,\;T=20$ & $900$  & $80$
   & \citep{geofno}\\
  Burgers    & $1024,\;T=20$           & $1000$ & $200$
  & \citep{takamoto2022pdebench}\\
  1D CFD      & $1024,\;T=20$           & $1800$ & $200$
  & \citep{takamoto2022pdebench}\\
  2D CFD      & $128\!\times\!128,\;T=10$ & $1800$ & $200$
  & \citep{takamoto2022pdebench}\\
  ReacDiff   & $1024,\;T=20$           & $1000$ & $200$
  & \citep{takamoto2022pdebench}\\
  \bottomrule
  \end{tabular}
\end{table}

For time-dependent PDEs in Table~\ref{tab:datasets}, we adopt one-step autoregressive rollouts, where predicted values are recursively fed back during both training and inference, except for Plasticity, which is modeled non-autoregressively by predicting the deformation field at a queried time step from a static spatial input field. Specifically, we use the previous 10 steps as input for Navier-Stokes, Burgers, 1D CFD, and reaction-diffusion, and the previous 5 steps as input for 2D CFD.

Additionally, for zero-shot super-resolution, we use the Darcy benchmark and the FNO implementation from the NeuralOperator library \citep{kossaifi2025librarylearningneuraloperators}, because FNO exhibits substantially stronger zero-shot super-resolution performance in this setup, making it a more competitive baseline for comparison against our SAFDNO.

\begin{table}[htbp]
    \centering
    \caption{Hyperparameters for SAFDNO and FNO.}
    \label{tab:hyperparams}
    \begin{tabular}{lccc|ccc}
    \toprule
    & \multicolumn{3}{c|}{FNO} & \multicolumn{3}{c}{SAFDNO} \\
    \cmidrule(r){2-4} \cmidrule(l){5-7}
    Dataset & Width & Layers & Modes & Width & Layers & Poles \\
    \midrule
    Plasticity     & / & / & / & 64 & 4 & $(16,16)$ \\
    Airfoil        & /  & / &  / & 64 & 4 & $(16,16)$ \\
    Pipe           & / & / & / & 64 & 4 & $(16,16)$ \\
    Navier-Stokes & 64  & 4 & $(32,32)$ & 64 & 4 & $(16,16)$ \\
    Darcy          & 128 & 4 & $(12,12)$ & 64 & 4 & $(32,32)$ \\
    Burgers        & 32  & 4 & 1024      & 32 & 4 & 48 \\
    1D CFD         & 32  & 4 & 1024      & 32 & 4 & 16 \\
    ReacDiff       & 32  & 4 & 1024      & 64 & 4 & 16 \\
    2D CFD         & 32  & 4 & $(32,32)$ & 32 & 4 & $(16,16)$ \\
    \bottomrule
    \end{tabular}
  \end{table}

\clearpage
\section{Additional Experiment Results}\label{app:additional_results}
This appendix reports additional results that complement Section~\ref{sec:main_results}, including efficiency comparisons between SAFDNO and FNO and visualizations of prediction performance on PDEs.
  
\subsection{Efficiency Analysis}\label{app:efficiency_analysis}

\begin{table}[htbp]
    \centering
    \caption{Efficiency comparison on the Darcy benchmark with a batch size of $1$. Training time is the average epoch time. Inference time and memory are measured per sample.}
    \label{tab:darcy_efficiency}
    \begin{tabular}{lcc}
      \toprule
      Metric & FNO & SAFDNO \\
      \midrule
      Params                     & 19.01M & 1.56M \\
      Train Time / Epoch (s)     & 5.79   & 25.49 \\
      Train Memory (MB)          & 741.76 & 541.11 \\
      Inference Time / Sample (ms) & 1.07 & 8.66 \\
      Inference Memory (MB)      & 209.00 & 115.59 \\
      Inference FLOPs (GFLOPs)   & 2.20   & 7.98 \\
      \bottomrule
    \end{tabular}
\end{table}

\begin{table}[htbp]
    \centering
    \caption{Efficiency comparison on the Navier-Stokes equation with a batch size of $6$. Training time is the average epoch time. Inference time and memory are measured per sample.}
    \label{tab:ns_efficiency}
    \begin{tabular}{lcc}
      \toprule
      Metric & FNO & SAFDNO \\
      \midrule
      Params                        & 33.59M  & 1.55M \\
      Train Time / Epoch (s)        & 7.45    & 43.30 \\
      Train Memory (MB)             & 2187.64 & 11974.76 \\
      Inference Time / Sample (ms)  & 10.59   & 76.04 \\
      Inference Memory (MB)         & 312.39  & 56.72 \\
      Inference FLOPs (GFLOPs)      & 3.16    & 36.29 \\
      \bottomrule
    \end{tabular}
\end{table}

\begin{table}[htbp]
    \centering
    \caption{Scaling effect of the number of Fourier modes and SAFD poles on the Darcy flow benchmark. We vary the number of modes or poles. Under the same setting, SAFDNO consistently achieves lower test error than FNO. Increasing the number of SAFDNO poles improves accuracy while keeping the trainable parameter count nearly unchanged, whereas increasing the number of FNO modes substantially increases the number of trainable parameters.}
    \label{tab:darcy_hp_sensitivity}
    \begin{tabular}{lccc}
      \toprule
      Model & poles/modes & Test rel.\ err.\ (\%) & Params (M) \\
      \midrule
      FNO    & (8, 8)    & 1.10 &   8.52 \\
      FNO    & (16, 16)  & 1.12 &  33.69 \\
      FNO    & (32, 32)  & 1.14 & 134.35 \\
      \midrule
      SAFDNO & (8, 8)    & 0.81 &   1.54 \\
      SAFDNO & (16, 16)  & 0.64 &   1.55 \\
      SAFDNO & (32, 32)  & \textbf{0.54} &   1.56 \\
      \bottomrule
    \end{tabular}
\end{table}

\clearpage
\subsection{Visualizations of Prediction Performance and Error Analysis}\label{app:performance_visual}

\begin{figure}[htbp]
    \centering
    \includegraphics[width=0.8\linewidth]{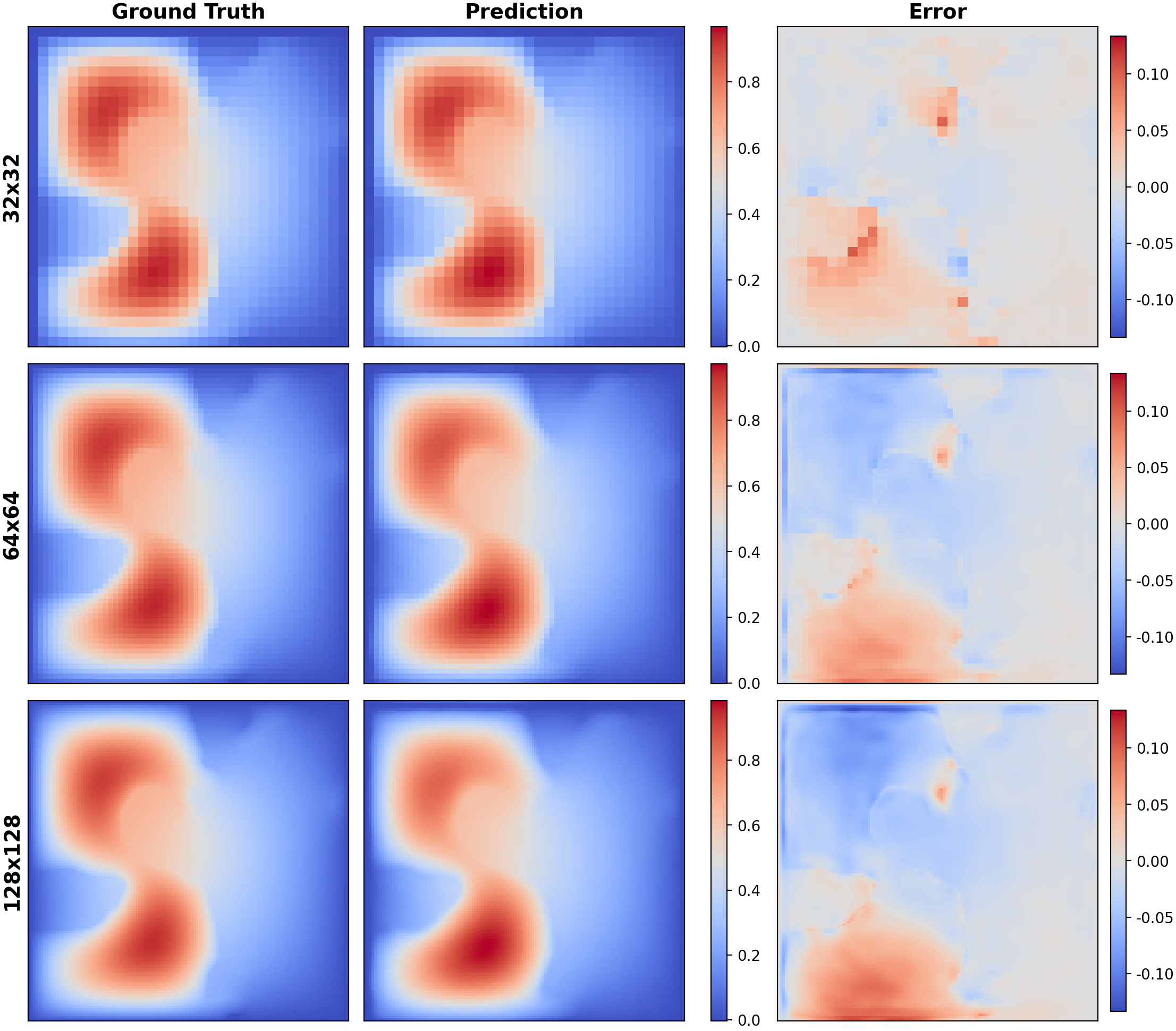}
    \caption{Zero-Shot Super-Resolution of SAFDNO on 2D Darcy Flow.}
    \label{fig:afdno_zeroshot_darcy}
\end{figure}

\begin{figure}[htbp]
    \centering
    \includegraphics[width=0.9\linewidth]{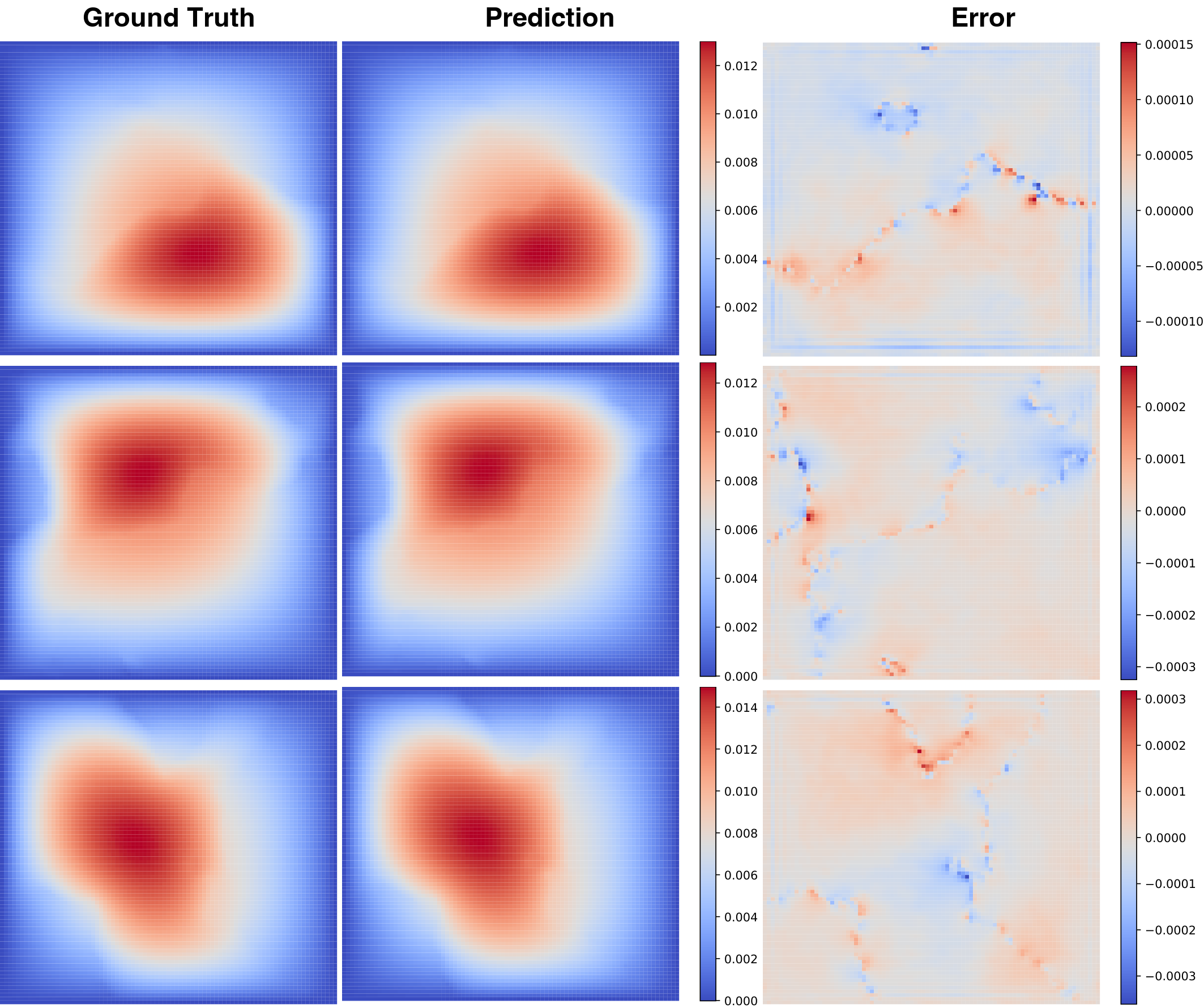}
    \caption{Performance of SAFDNO on Darcy}
    \label{fig:SAFDNO_darcy}
\end{figure}

\begin{figure}[htbp]
    \centering
    \includegraphics[width=0.8\linewidth]{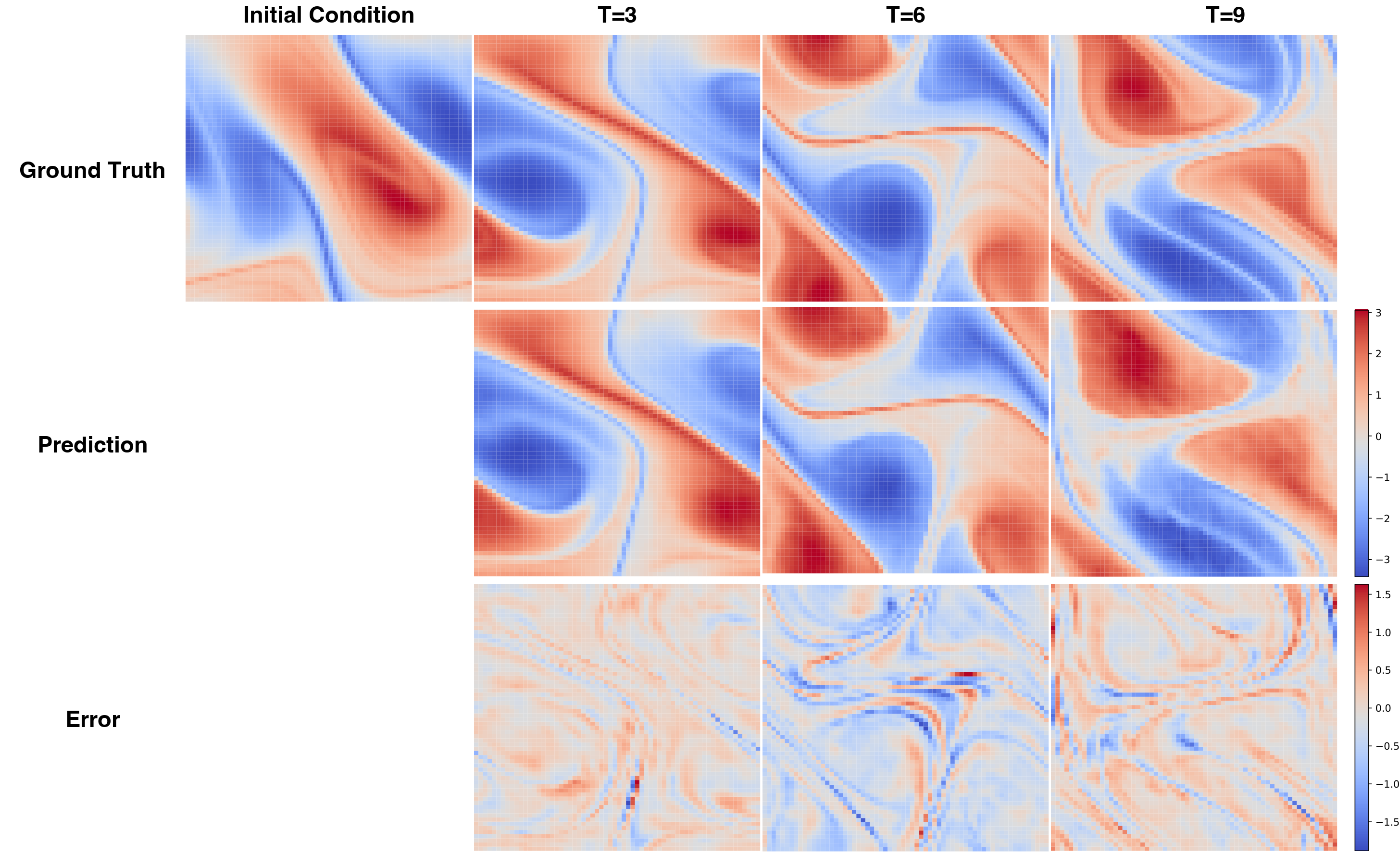}
    \caption{Performance of SAFDNO on Navier-Stokes}
    \label{fig:SAFDNO_ns}
\end{figure}

\begin{figure}[htbp]
    \centering
    \includegraphics[width=0.8\linewidth]{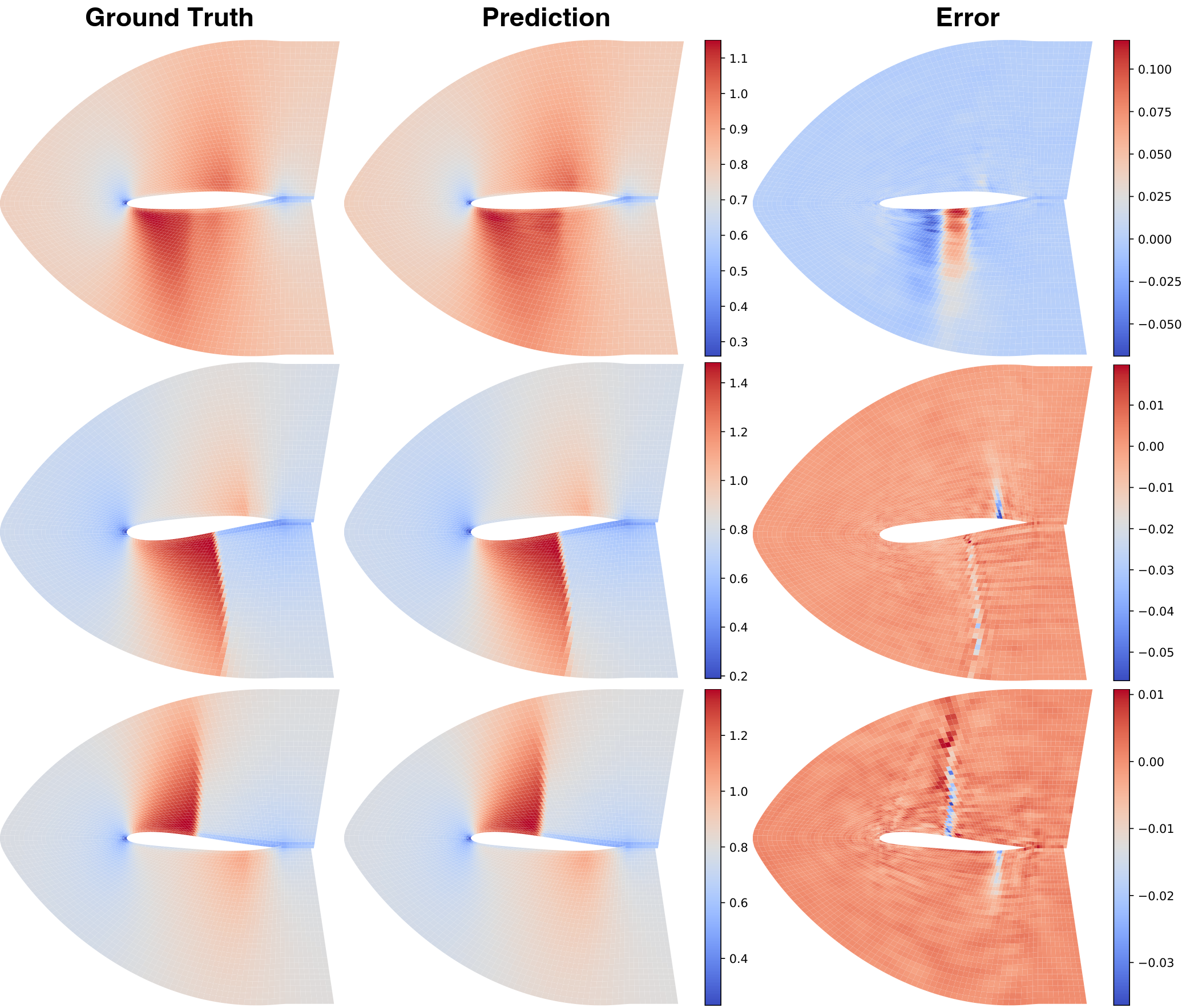}
    \caption{Performance of SAFDNO on Airfoil}
    \label{fig:SAFDNO_airfoil}
\end{figure}

\begin{figure}[htbp]
    \centering
    \includegraphics[width=0.8\linewidth]{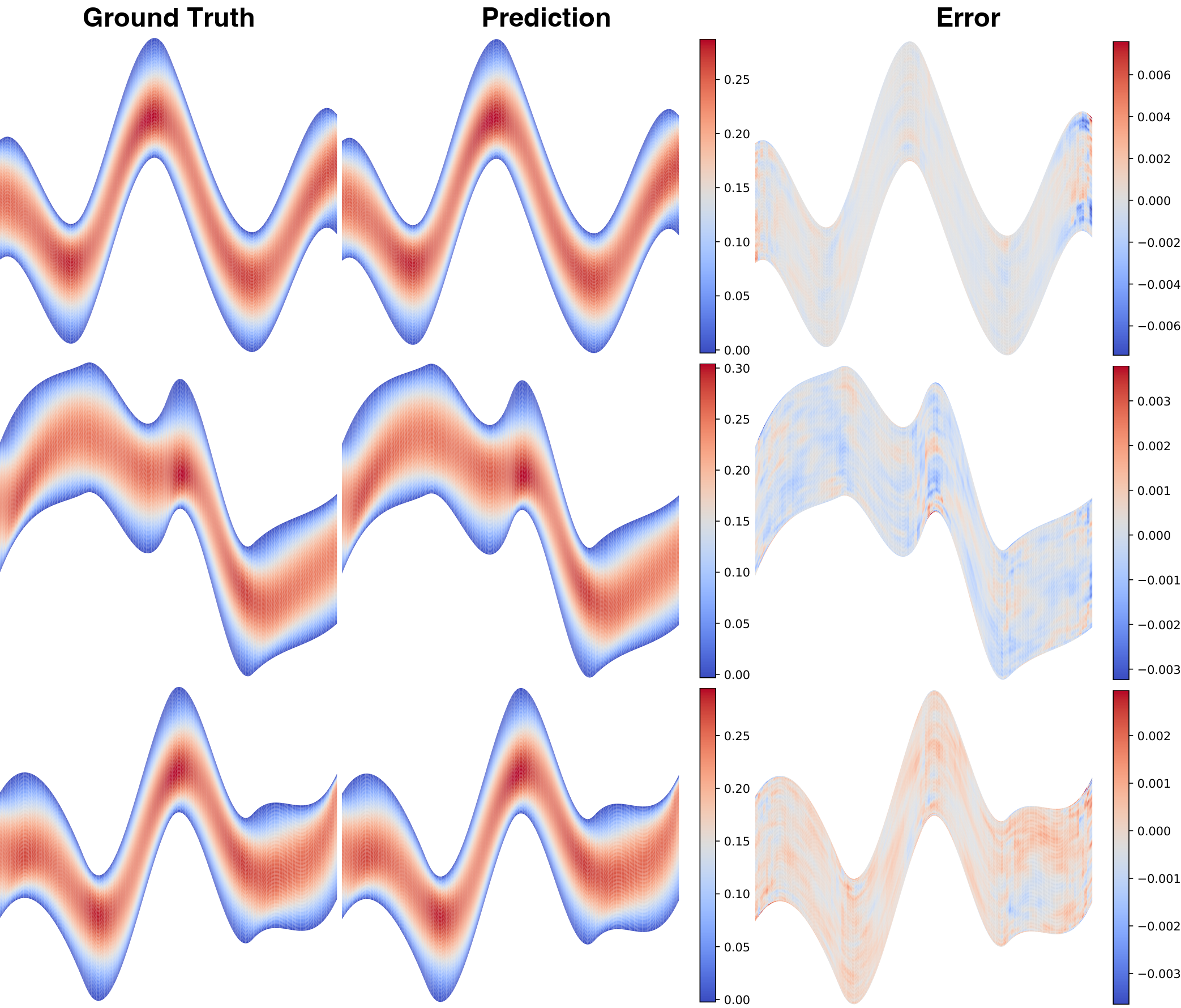}
    \caption{Performance of SAFDNO on Pipe}
    \label{fig:SAFDNO_pipe}
\end{figure}

\begin{figure}[htbp]
    \centering
    \includegraphics[width=0.8\linewidth]{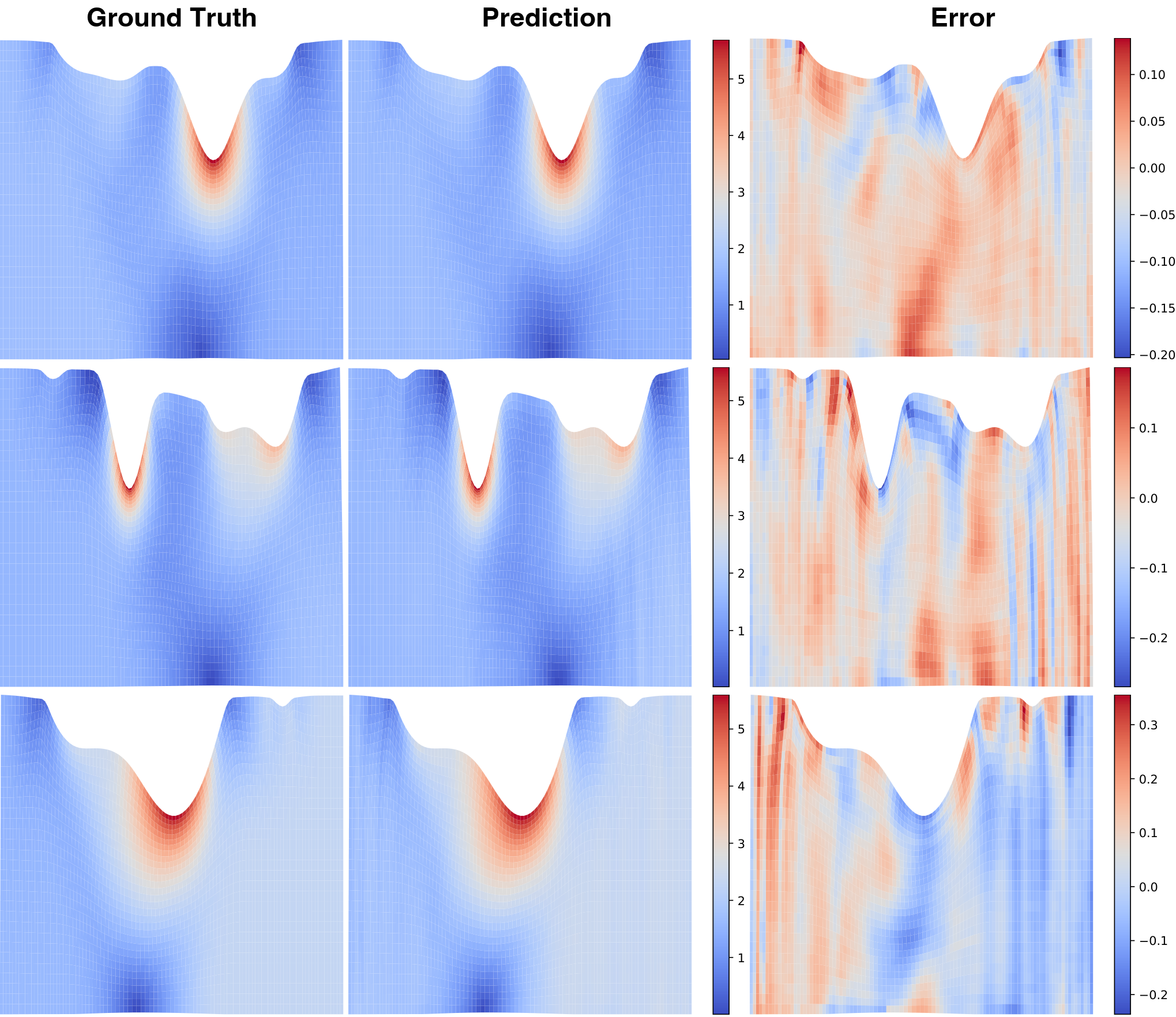}
    \caption{Performance of SAFDNO on Plasticity}
    \label{fig:SAFDNO_pla}
\end{figure}

\begin{figure}[htbp]
    \centering
    \begin{subfigure}{0.48\linewidth}
        \centering
        \includegraphics[width=\linewidth]{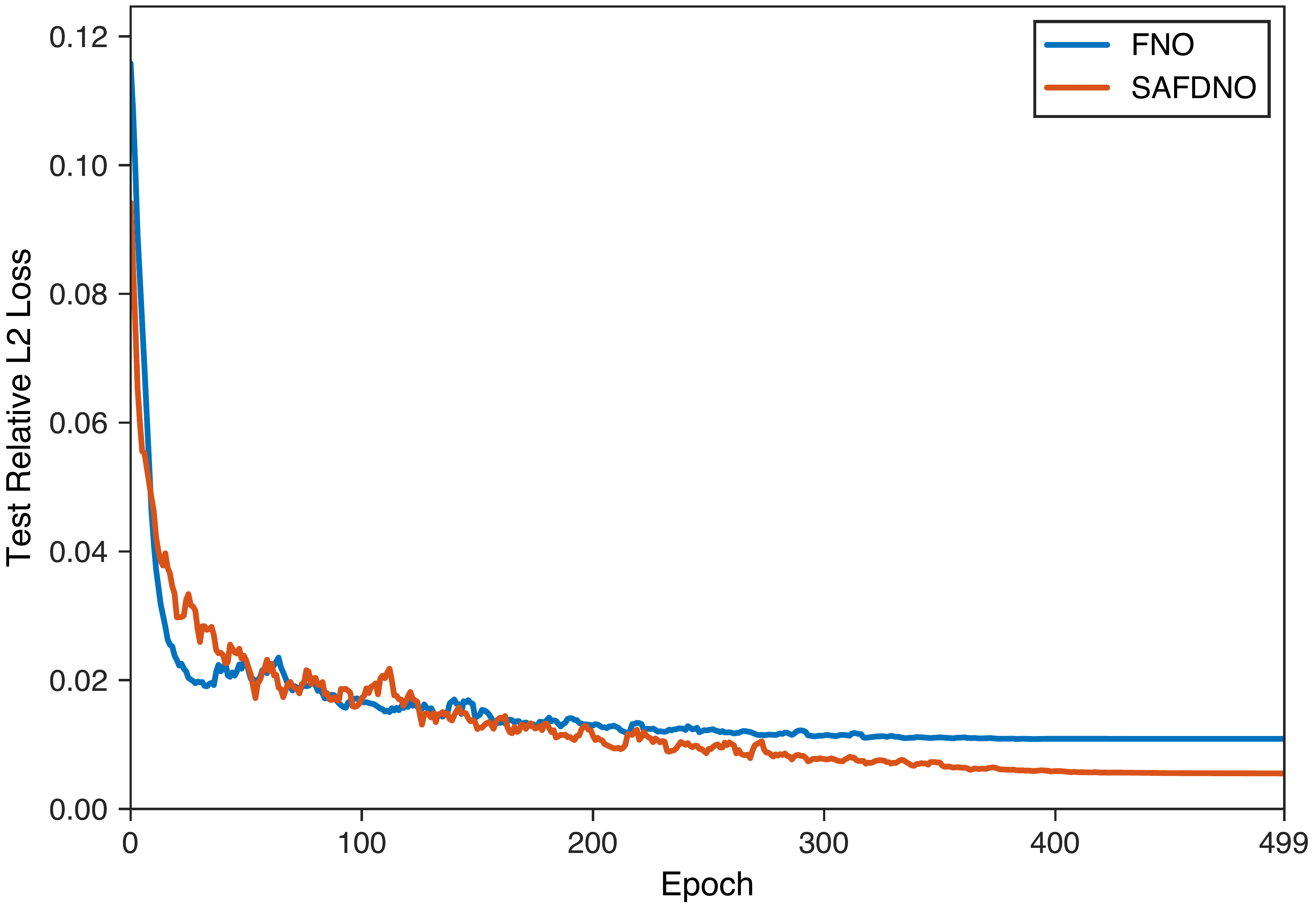}
        \caption{Darcy}
        \label{fig:loss_compare_darcy}
    \end{subfigure}\hfill
    \begin{subfigure}{0.48\linewidth}
        \centering
        \includegraphics[width=\linewidth]{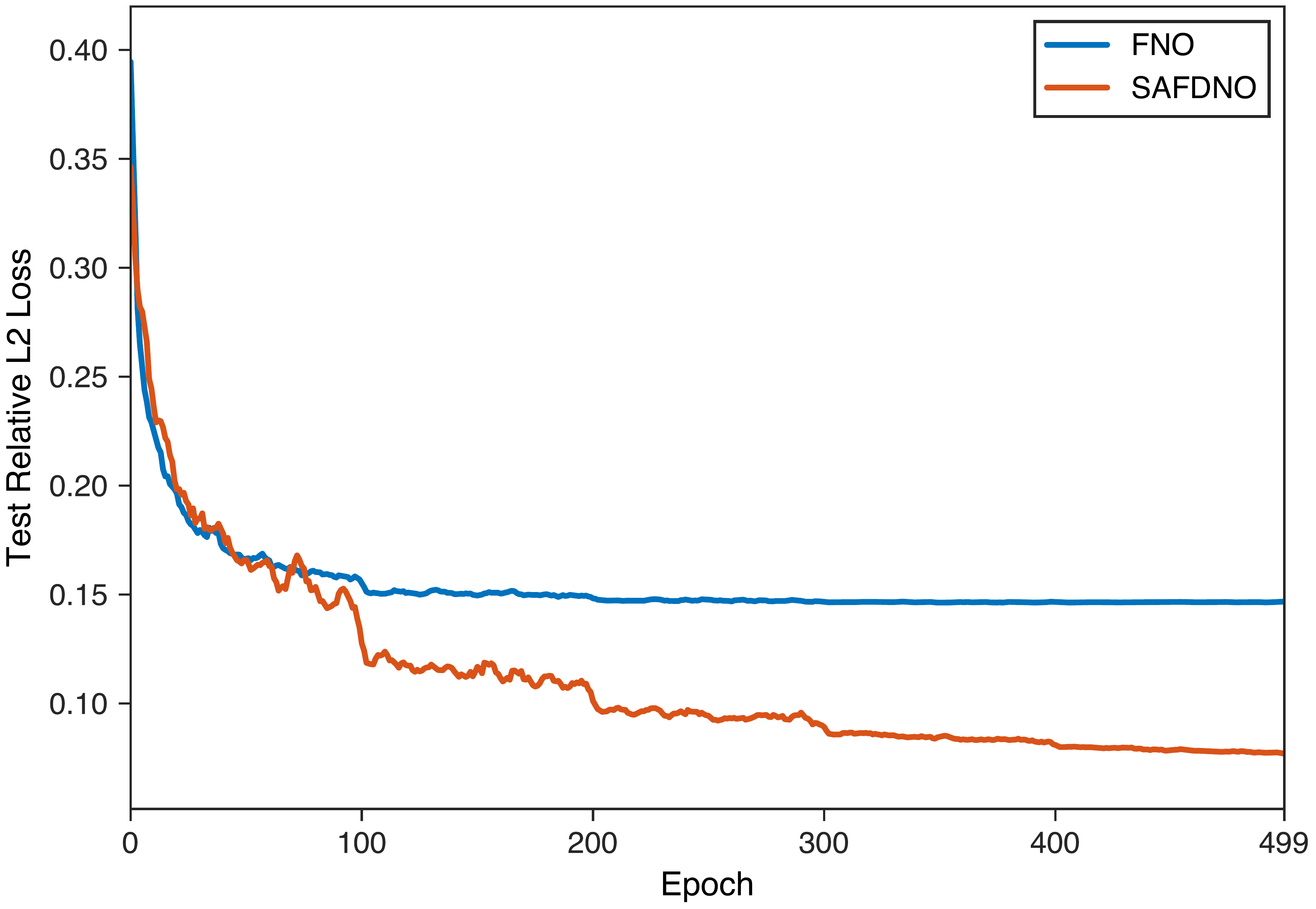}
        \caption{Navier-Stokes}
        \label{fig:loss_compare_ns}
    \end{subfigure}
    \caption{Test relative $\ell_2$ error - SAFDNO vs FNO}
    \label{fig:loss_compare_darcy_ns}
\end{figure}






\end{document}